\documentclass[11pt]{article}
\usepackage[T1]{fontenc}
\usepackage{lmodern}
\usepackage{cfr-lm}
\newcommand{\HiPaN}{\textsc{HiPaN}}
\newcommand{\Vpunn}{\textsc{v‑PuNN}}
\newcommand{\vpunn}{\Vpunn}
\newcommand{\vpunns}{\textsc{v‑PuNN}s}

\newcommand{\Turl}{\textsc{Turl}}

\usepackage[utf8]{inputenc}
\usepackage[a4paper,margin=2.8cm]{geometry}

\usepackage{amsmath,amssymb,amsthm,amsfonts,mathtools}
\usepackage{siunitx}                       % clean numbers / units
\usepackage{booktabs,tabularx,makecell,multirow}
\usepackage{graphicx,subcaption,placeins}
\usepackage{enumitem}
\usepackage{float}
\usepackage{bbm}                           % \mathbbm{1}
\newcommand{\pref}[2]{\operatorname{pref}_{#1}(#2)}
\usepackage{bm}                            % bold math symbols
\usepackage[table]{xcolor}                 % color for tables
\usepackage{mathrsfs}                      % \mathscr
\usepackage{microtype}                     % kerning, protrusion
\usepackage{algorithm}                     % floating algorithm env.
\usepackage{algpseudocode}                 % pseudocode layout

\usepackage{tikz}
\usetikzlibrary{arrows.meta,positioning,calc,fit,shadows,matrix}
\usepackage{natbib}                         % author–year citations
\usepackage{hyperref}                       % keep LAST
\hypersetup{
    colorlinks=true,
    linkcolor=blue!60!black,
    citecolor=green!50!black,
    urlcolor=orange!70!black
}

\newcommand{\Q}{\mathbb Q}

\newcommand{\Z}{\mathbb Z}

\newcommand{\R}{\mathbb R}

\DeclareMathOperator*{\argmin}{arg\,min}

\DeclareMathOperator{\val}{val}
\DeclareMathOperator{\depth}{depth}
\DeclareMathOperator{\E}{\mathbb E}
\theoremstyle{plain}
\newtheorem{theorem}{Theorem}[section]
\newtheorem{lemma}{Lemma}[section]
\newtheorem{corollary}{Corollary}[section]
\newtheorem{proposition}{Proposition}[section]
\newtheorem{definition}{Definition}[section]

\theoremstyle{remark}
\newtheorem*{remark}{Remark}
\newtheorem*{observation}{Observation}

\setlist[itemize]{nosep,leftmargin=1.2em}
\setlist[enumerate]{nosep,leftmargin=1.4em}

\author{Gnankan Landry Regis N'guessan}
\begin{document}
%\maketitle

\begin{center}
 \rule{\linewidth}{0.5pt} \\[0.4cm]
 {\LARGE \bfseries v-PuNNs: van der Put Neural Networks\texorpdfstring{\\}{ }[0.1cm]
for Transparent Ultrametric Representation\\[0.2cm] Learning} \\[0.4cm]
 \rule{\linewidth}{0.5pt}
 \end{center}

\begin{center}
\textbf{Gnankan Landry Regis N'guessan$^{1,2,3}$} \\[0.2cm]
$^1$Axiom Research Group \\
$^2$Department of Applied Mathematics and Computational Science, \\
    The Nelson Mandela African Institution of Science and Technology (NM-AIST), Arusha, Tanzania \\
$^3$African Institute for Mathematical Sciences (AIMS), Research and Innovation Centre (RIC), Kigali, Rwanda
\end{center}

\vspace{1cm}

% ------------------------------------------------------------------
% Abstract
% ------------------------------------------------------------------
\begin{abstract}
Conventional deep learning models embed data in Euclidean space $\mathbb{R}^d$, a poor fit for strictly hierarchical objects such as taxa, word senses, or file systems, often inducing high distortion. We address this geometric mismatch with \textbf{van der Put Neural Networks (v-PuNNs)}, the first architecture to operate natively in ultrametric $p$-adic space, where neurons are characteristic functions of $p$-adic balls in $\mathbb{Z}_p$, and its practical implementation \textbf{Hierarchically-Interpretable $p$-adic Network (HiPaN)}. Grounded in our \textbf{Transparent Ultrametric Representation Learning (TURL)} principle, v-PuNNs are white-box models where every weight is a $p$-adic number, providing exact subtree semantics. Our new \textbf{Finite Hierarchical Approximation Theorem} proves that a depth-$K$ v-PuNN with $\sum_{j=0}^{K-1}p^{\,j}$ neurons can universally approximate any function on a $K$-level tree. Because gradients vanish in this discrete space, we introduce \textbf{Valuation-Adaptive Perturbation Optimization (VAPO)}, with a fast deterministic variant \textbf{GIST-VAPO} and a moment‑based one \textbf{Adam‑VAPO}. Our CPU-only implementations set new state-of-the-art results on three canonical benchmarks: on \textbf{WordNet nouns} (52,427 leaves), we achieve 99.96\% leaf accuracy in under 17 minutes; on \textbf{Gene Ontology molecular function} (27,638 proteins), we attain 96.9\% leaf and 100\% root accuracy in 50 seconds; and on \textbf{NCBI Mammalia} (12,205 taxa), the learned metric correlates with ground-truth taxonomic distance at a Spearman $\rho = -0.96$, surpassing all Euclidean and tree-aware baselines. Crucially, the learned metric is perfectly ultrametric, with zero triangle violations. We analyze the fractal and information-theoretic properties of the space and demonstrate the framework's generality by deriving structural invariants for quantum systems (\textbf{HiPaQ}) and discovering latent hierarchies for generative AI (\textbf{Tab-HiPaN}). v-PuNNs therefore bridge number theory and deep learning, offering exact, interpretable, and efficient models for hierarchical data.
\end{abstract}


% ------------------------------------------------------------------
% 1. Introduction
% ------------------------------------------------------------------
\section{Introduction}

Deep learning has achieved unprecedented success by embedding complex data
into the latent spaces of neural networks. By default, these spaces are
Euclidean $(\mathbb R^{d})$, a choice so fundamental that it is rarely
questioned. However, a vast portion of the world's most valuable information, from
the taxonomic tree of life and the semantic structure of language to file
systems and organizational charts, is not unstructured, but organized into
strict, nested hierarchies. Forcing these inherently hierarchical data into a
Euclidean space creates a fundamental geometric mismatch, resulting in
high-distortion embeddings where structural relationships are obscured~\cite{linial1995geometry}, and learned features lack clear, interpretable meaning. Although recent advances in hyperbolic geometry offer a promising alternative, they still rely on
continuous approximations of fundamentally discrete structures. We posit that the natural geometry for hierarchical data is neither Euclidean nor hyperbolic, but ultrametric~\cite{rammal1986ultrametricity,murtagh2004ultrametric}. The canonical space for this geometry is the field of $p$-adic numbers $\mathbb Q_{p}$, where the distance between two points is determined by the depth of their lowest common ancestor~\cite{gouvea1997p}. In this paper, we close the geometric gap by introducing v-PuNNs, a novel class of architectures native to this space.

\begin{figure}[H] % requires \usepackage{float}
  %----- local spacing tweaks (begin group) -----
  {
    \captionsetup{skip=2pt}            % tighten graphic–caption gap
    \setlength\abovecaptionskip{1pt}   % space above caption
    \setlength\belowcaptionskip{0pt}   % space below caption
    \setlength\intextsep{4pt}          % space above/below the whole float
    \centering
    \includegraphics[width=0.88\linewidth]{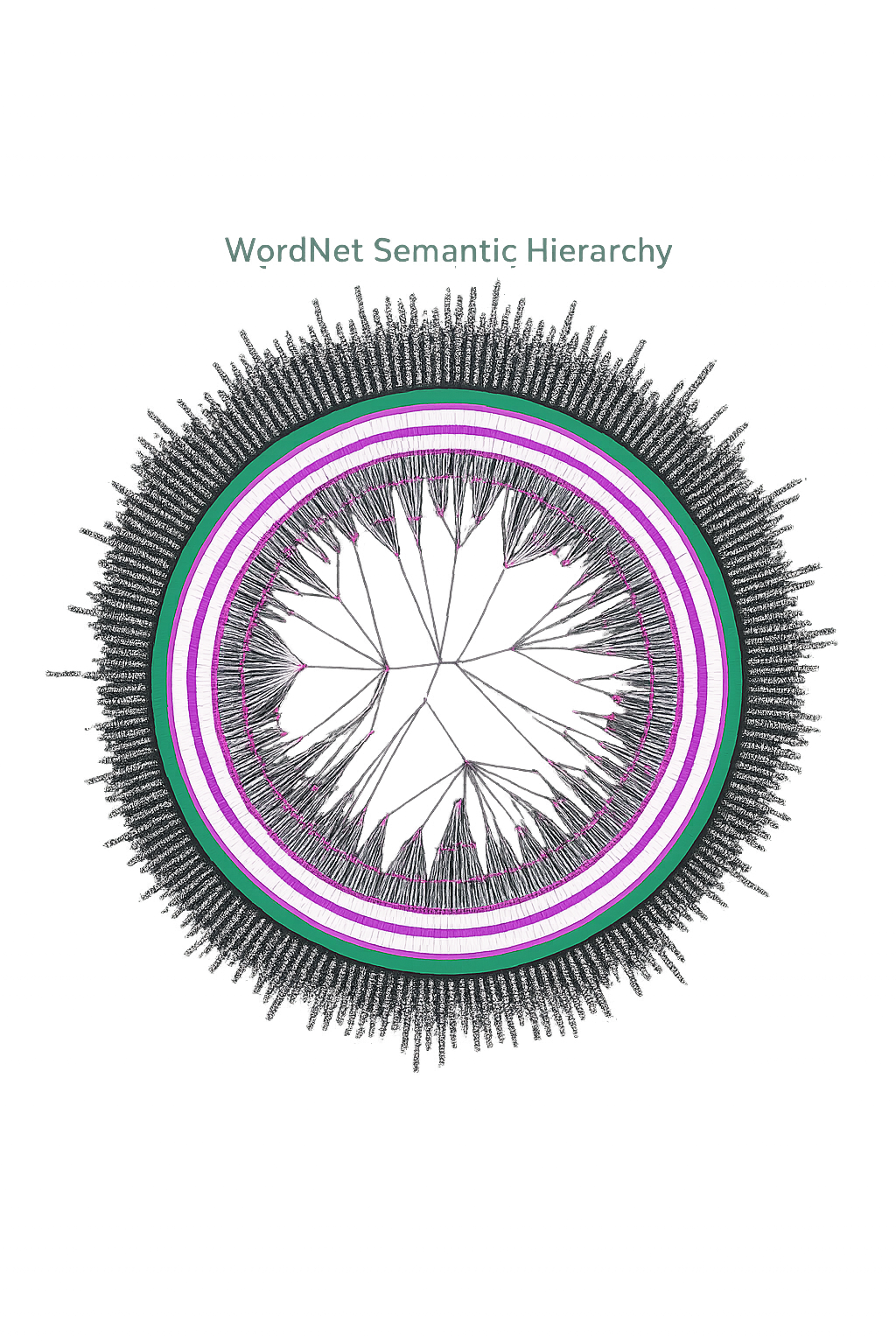}
    \caption{Force‑directed Euclidean layout of the full WordNet noun hierarchy. Its visual clutter motivates an ultrametric treatment.}
    \label{fig:wordnet-raw}
  }
  %----- local spacing tweaks (end group) -------
\end{figure}

\subsection*{Key Contributions}
\begin{itemize}
  \item \textbf{v-PuNNs}: The first class of neural architectures whose neurons are
        characteristic functions of $p$-adic balls, enabling lossless,
        white-box representation of hierarchies.
  \item \textbf{Transparent Ultrametric Representation Learning (TURL)}: Every parameter of the model is a $p$-adic number with a direct structural interpretation.
  \item \textbf{Finite Hierarchical Approximation Theorem}: A depth-$K$
        v-PuNN with exactly $\sum_{j=0}^{K-1}p^{j}$ neurons universally approximates any function on a $K$-level hierarchy.
  \item \textbf{Valuation-Adaptive Perturbation Optimization (VAPO)}: A new class of efficient derivative-free optimizers for discrete non-Archimedean spaces.
  \item \textbf{CPU-level state-of-the-art}: HiPaN(Adam-VAPO) trains the 52 k-leaf
        WordNet to $\ge 99.96\%$ accuracy in \(\sim\)16 min on a 32 GB CPU and reaches $100\%$ accuracy on Gene Ontology (27 k leaves) in 50 s, outperforming all Euclidean and tree‑aware baselines we tested.
  \item \textbf{First direct validation of a learned ultrametric}: Learned
        $p$-adic distances correlate with the true taxonomic depth ($\rho=-0.96$) while keeping the expected calibration error $\le0.7\%$.
  \item \textbf{Generality beyond classification}: HiPaQ and Tab-HiPaN show
the reach of the framework in physics, mathematics, and controllable data generation.
\end{itemize}

% ------------------------------------------------------------------
% Mini-Glossary (place after Introduction or as §1.1)
% ------------------------------------------------------------------
\subsection*{Glossary of Abbreviations}

\begin{tabular}{p{2.4cm} p{10.2cm}}
\toprule
\textbf{Acronym} & \textbf{Meaning} \\
\midrule
v-PuNN   & \textbf{v}an der \textbf{Pu}t \textbf{N}eural \textbf{N}etwork\\
HiPaN            & \textbf{H}ierarchically‑\textbf{I}nterpretable $p$‑\textbf{a}dic \textbf{N}etwork\\
VAPO             & \textbf{V}aluation‑\textbf{A}daptive \textbf{P}erturbation \textbf{O}ptimisation\\
Adam‑VAPO & \textbf{A}dam‑style \textbf{V}aluation‑\textbf{A}daptive \textbf{P}erturbation \textbf{O}ptimisation.\\
GIST‑VAPO        & \textbf{G}reedy \textbf{I}nteger \textbf{S}tep \textbf{T}uning combined with VAPO.\\
HiPaN‑DS\\ (GIST‑VAPO)        & \textbf{D}eterministic‑\textbf{S}earch variant of HiPaN trained with \textbf{G}reedy \textbf{I}nteger \textbf{S}tep \textbf{T}uning plus VAPO.\\
HiPaN\\ (Adam‑VAPO)        & Same architecture trained with \textbf{V}aluation‑\textbf{A}daptive \textbf{P}erturbation \textbf{O}ptimisation using Adam‑style updates.\\
TURL             & \textbf{T}ransparent \textbf{U}ltrametric \textbf{R}epresentation \textbf{L}earning.\\
HiPaQ            & \textbf{H}ierarchical $p$‑\textbf{a}dic \textbf{Q}uantifier.\\
Tab‑HiPaN        & v‑PuNN used to discover latent hierarchies in tabular data for controllable generation.\\
ECE              & \textbf{E}xpected \textbf{C}alibration \textbf{E}rror.\\
LCA              & \textbf{L}owest \textbf{C}ommon \textbf{A}ncestor\\
VC               & \textbf{V}apnik–\textbf{C}hervonenkis dimension\\
RMSE             & \textbf{R}oot‑\textbf{M}ean‑\textbf{S}quared \textbf{E}rror.\\
MMD              & \textbf{M}aximum \textbf{M}ean \textbf{D}iscrepancy.\\
MSE              & \textbf{M}ean \textbf{S}quared \textbf{E}rror.\\
CE               & \textbf{C}ross \textbf{E}ntropy.\\
Brier            & Brier score\\
UMAP             & \textbf{U}niform \textbf{M}anifold \textbf{A}pproximation and \textbf{P}rojection\\
VAE              & \textbf{V}ariational \textbf{A}uto‑\textbf{E}ncoder.\\
c‑VAE            & \textbf{C}onditional VAE.\\
MLP              & \textbf{M}ulti‑\textbf{L}ayer \textbf{P}erceptron.\\
GNN              & \textbf{G}raph \textbf{N}eural \textbf{N}etwork.\\
KL               & \textbf{K}ullback-\textbf{L}eibler divergence.\\
SHAP             & \textbf{SH}apley \textbf{A}dditive ex\textbf{P}lanations plots.\\
LIME             & \textbf{L}ocal \textbf{I}nterpretable \textbf{M}odel‑agnostic \textbf{E}xplanations.\\
GO              & \textbf{G}ene \textbf{O}ntology.\\
GPU              & \textbf{G}raphics \textbf{P}rocessing \textbf{U}nit.\\
\bottomrule
\end{tabular}

% ================================================================
\section{Related Work and Methodological Gaps}
\label{sec:related_work}
% ================================================================

Learning faithful representations of hierarchical data remains a persistent challenge. Prior research has approached this problem from several major directions: Euclidean embeddings, hyperbolic geometry, graph neural networks, and ultrametric methods. Each addresses part of the puzzle but leaves critical limitations unresolved, which our work confronts head on.

% ---------------------------------------------------------------
\subsection{Opaque Embeddings in Euclidean Space}
% ---------------------------------------------------------------

Standard deep learning pipelines often embed items as vectors in $\mathbb{R}^{d}$, recovering tree structure implicitly through learned distances. While any non-path tree metric incurs some distortion in Euclidean space, this distortion is remarkably low, scaling as $O(\log \log n)$ for trees with $n$ nodes \cite{linial1995geometry}. However, the fundamental limitation lies not in distortion magnitude but in the polynomial volume growth of Euclidean space, which struggles to accommodate the exponential expansion of complex hierarchies without compromising structural fidelity. Consequently, large taxonomies forced into low-dimensional Euclidean space exhibit significant information loss and metric distortion, as empirically demonstrated on datasets like WordNet \cite{nickel2017poincare}. Moreover, Euclidean coordinates are opaque: they lack explicit semantic mapping to hierarchical properties. Unlike hyperbolic models where vector norms represent depth, Euclidean embeddings offer no interpretable link between coordinates and hierarchical levels \cite{tifrea2018poincare}. In summary, while ubiquitous, Euclidean embeddings do not intrinsically account for latent hierarchies, yielding representations that are unintelligible with respect to the original taxonomy \cite{nickel2017poincare}.

% ---------------------------------------------------------------
\subsection{Continuous Approximations in Hyperbolic Models}
% ---------------------------------------------------------------

To better model hierarchies, researchers use hyperbolic geometry as a continuous analogue of trees, leveraging its exponential volume growth to mirror combinatorial tree expansion \cite{nickel2017poincare,sala2018representation}. Any finite tree can be embedded in hyperbolic space with arbitrarily low distortion, though this requires scaling the tree's metric by a factor dependent on precision and tree structure \cite{sarkar2012low}. Poincaré embeddings significantly outperform Euclidean models on hierarchical datasets \cite{nickel2017poincare}, and Lorentz (hyperboloid) variants improve optimization efficiency and embedding quality via Riemannian SGD \cite{nickel2018learning}.  

However, hyperbolic approaches face critical limitations. First, their parameters remain opaque: while vector norms correlate with depth, individual coordinates lack direct subtree semantics, requiring post-hoc interpretation \cite{tifrea2018poincare}. Second, training is GPU-centric and computationally intensive due to Riemannian optimization, demanding careful hyperparameter tuning even with closed-form geodesics \cite{nickel2018learning}. Third, top-level (root) nodes are prone to drift away from the origin during optimization, distorting the highest-level structure \cite{sala2018representation}. Fourth, the calibration of prediction probabilities, how well continuous scores reflect true correctness likelihoods, remains underexplored, creating reliability gaps for downstream tasks.  

These issues stem from a core trade-off: achieving low distortion in hyperbolic space necessitates high numerical precision and exacerbates optimization instability, particularly for deep hierarchies.

% ---------------------------------------------------------------
\subsection{Graph Neural and Transformer Models}
% ---------------------------------------------------------------

Graph Neural Networks (GNNs) and Graph Transformers treat hierarchies as general graphs, learning representations via neighborhood aggregation. While powerful for capturing local connectivity, they ignore global ultrametric constraints inherent to strict hierarchies. Architecturally, they are opaque: attention weights and message-passing functions obscure explicit hierarchical relationships. Computationally, they are expensive; training on large graphs typically requires substantial GPU resources even with sampling optimizations \cite{shehzad2024graphsurvey}, which conflicts with our CPU-frugal goals. Most critically, they learn statistical patterns without geometric guarantees for preserving hierarchical properties.

% ---------------------------------------------------------------
\subsection{Prior Ultrametric and \texorpdfstring{$p$-adic}{p-adic} Attempts}
% ---------------------------------------------------------------

Ultrametric spaces, where every triangle is isosceles with the long side,
naturally encode tree hierarchies by enforcing shared nearest ancestors
\cite{rammal1986ultrametricity}. Early work leveraged $p$-adic systems for hierarchical representation \cite{gouvea1997p}, and introduced neural models such as
$p$-adic Hopfield networks and perceptrons \cite{albeverio1999padicNN,khrennikov2000learning,zambranoluna2022padicCNN}.
Despite their theoretical appeal, these methods did not scale beyond toy datasets, although very recent work has begun to train $p$-adic CNNs on full-resolution images \cite{zuniga2024padicCNN}. A scarcity of scalable “non-Archimedean” adaptations of mainstream machine-learning algorithms persists, and classical agglomerative ultrametric clustering remains $O(n^{2})$ \cite{murtagh2004ultrametric}. Complementary vision research now optimizes an ultrametric loss end-to-end for hierarchical image segmentation \cite{lapertot2024ultrametric} and even learns ultrametric feature fields for 3-D scene hierarchies \cite{he2024ultrametricFF}, but none of these works addresses prediction
calibration or offers a principled, scalable, fully discrete architecture.

% ---------------------------------------------------------------
\subsection{Derivative-Free Optimization in Discrete Spaces}
% ---------------------------------------------------------------

Given the non-differentiability of discrete hierarchical representations, derivative-free optimizers (e.g., CMA-ES, Nelder-Mead) are theoretically relevant. However, they ignore ultrametric structure and scale poorly: CMA-ES struggles in high dimensions \cite{shimizu2021cmaes}, Nelder-Mead deteriorates in high-dimensional spaces \cite{lagarias1998convergence}, and Simulated Annealing has limited success in very high dimensions \cite{audet2017derivative}. These methods treat the search space as unstructured, converging slowly without exploiting $p$-adic valuation properties. Consequently, they remain impractical for hierarchies with tens of thousands of nodes.

\subsection{Remaining Gaps and How v-PuNNs Close Them}

\begin{table}[H]
\centering\small
\begin{tabular}{p{0.37\textwidth} p{0.58\textwidth}}
\toprule
Unresolved Issue in Prior Work & How v-PuNNs Respond\\
\midrule
Opaque models with no geometric semantics &
Every weight is a $p$-adic integer; its prefix equals an explicit subtree.\\
\addlinespace
GPU-centric and computationally heavy pipelines &
Full WordNet trains on a 32\,GB CPU in \(\sim\) 16\,min; Gene Ontology in 50\,s.\\
\addlinespace
Lack of guaranteed structural fidelity &
The learned ultrametric is lossless; zero triangle-inequality violations are observed.\\
\addlinespace
No principled, large-scale ultrametric learning &
Scales to 52\,k leaves (WordNet) and 27\,k terms (GO) with 19-digit depth.\\
\addlinespace
Optimizers that ignore valuation structure &
Introduces VAPO, digit-aware, valuation-adaptive methods converging in minutes on CPU.\\
\bottomrule
\end{tabular}
\end{table}

By addressing these gaps, v-PuNNs provide a unified solution that combines the interpretability and exactness of ultrametric representations with the scalability and accuracy expected of modern deep learning. In the following sections, we detail the v-PuNN architecture and training procedure, and demonstrate its performance on large-scale hierarchical benchmarks, establishing a new state of the art for faithful hierarchical representation learning.

\begin{table*}[!htbp]
\centering
\caption{v-PuNN vs. competing architectures across multiple axes.}
\label{tab:comparison}
\renewcommand{\arraystretch}{1.2}
\resizebox{\textwidth}{!}{%
\begin{tabular}{|p{3cm}|p{3.4cm}|p{3.5cm}|p{3.5cm}|p{4.2cm}|}
\hline
\textbf{Feature} & \textbf{Euclidean Models} (e.g., FFN, MLP) & \textbf{Hyperbolic Models} (e.g., Poincaré) & \textbf{Ensemble Trees} (e.g., XGBoost, LGBM) & \textbf{v-PuNN / HiPaN (Ours)} \\
\hline
Core Geometry & Flat $\mathbb{R}^d$ & Continuous $\mathbb{H}^d$, tree-approximating & Axis-aligned partitions in $\mathbb{R}^d$ & Discrete ultrametric $\mathbb{Q}_p$ matching tree topology \\
\hline
Interpretability & Opaque vector embeddings & Partial; depth can correlate with norm & Partial (decision paths traceable) & Fully transparent; parameters map to explicit subtrees \\
\hline
Structural Fidelity & High distortion & Low distortion but not exact & No isometry to hierarchy & Perfect isometry; hierarchy preserved exactly \\
\hline
Performance & Moderate on hierarchy tasks & Very strong on general tree-structured data & High per-digit accuracy; weak hierarchy modeling & State-of-the-art at all depths; strong correlation with structure \\
\hline
Computational Cost & Low & High (GPU \& Riemannian ops) & High memory & CPU-only; fast training; low memory \\
\hline
Optimization & SGD, Adam & Riemannian gradient descent & Gradient boosting & Derivative-free (VAPO) tailored to discrete trees \\
\hline
Key Innovation & Universal approximation & Exponential volume growth & Strong ensembling for classification & Finite Hierarchical Approximation Theorem + van der Put neural basis \\
\hline
Best For & Generic learning tasks & Graphs, social nets, tree-like data & Structured prediction, tabular data & Strict hierarchies, biological trees, semantic taxonomies \\
\hline
\end{tabular}%
}
\end{table*}

\FloatBarrier

% ===============================================================
% TikZ: Landscape of Geometries for Hierarchical Representation
% (adjusted labels, legend spacing, annotation position)
% ===============================================================
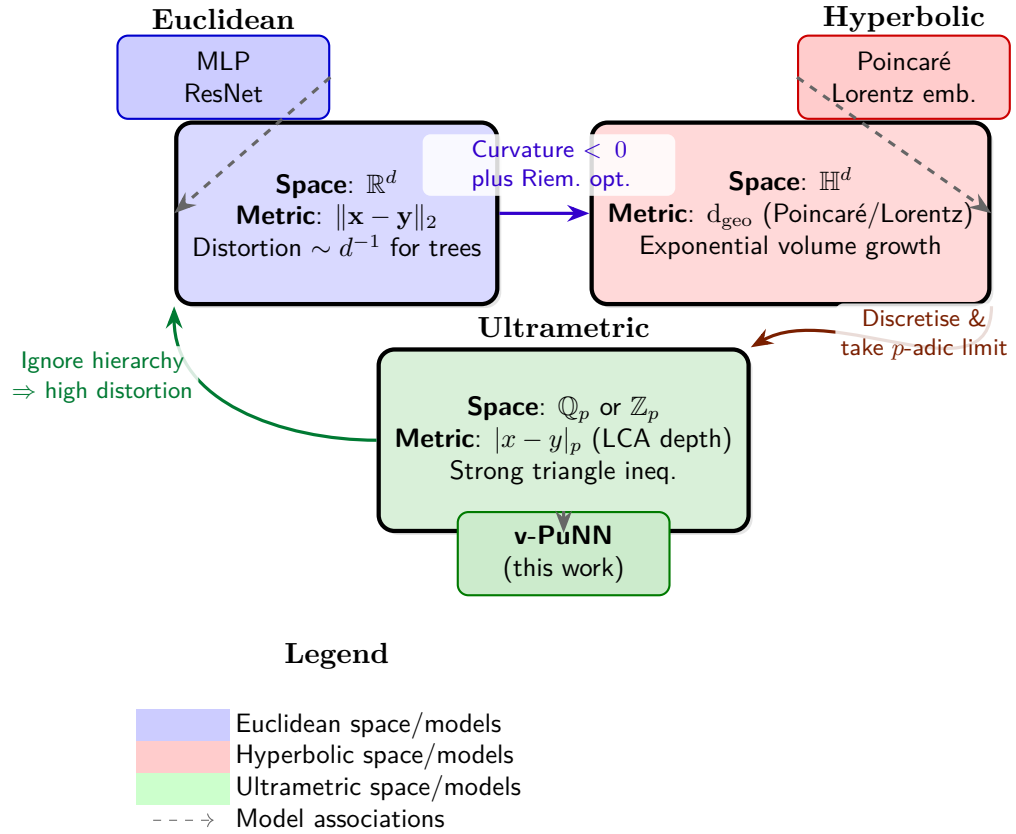
\begin{figure*}[!htbp]
\centering
\begin{tikzpicture}[
  font=\small\sffamily,
  every node/.style={align=center},
  mainbox/.style={rectangle,rounded corners=6pt,draw=black,ultra thick,fill=#1!15,
              minimum width=3.8cm,minimum height=2.4cm,inner sep=6pt,
              drop shadow={opacity=0.2,shadow xshift=1.5pt,shadow yshift=-1.5pt}},
  methodbox/.style={rectangle,rounded corners=4pt,draw=#1!80!black,thick,fill=#1!20,
              minimum width=2.8cm,minimum height=1.1cm,inner sep=3pt},
  arr/.style={->,line width=1.2pt,>=Stealth},
  annotation/.style={font=\footnotesize\sffamily,fill=white,fill opacity=0.85,
                     text opacity=1,inner sep=2pt,rounded corners=3pt}
]

% --- Main geometry boxes ----------------------------------------
\node[mainbox=blue] (eucl) at (0,0)
      {\textbf{Space}: $\mathbb R^{d}$\\
       \textbf{Metric}: $\lVert\mathbf x-\mathbf y\rVert_{2}$\\
       Distortion $\sim d^{-1}$ for trees};

\node[mainbox=red] (hyp) at (6,0)
      {\textbf{Space}: $\mathbb H^{d}$\\
       \textbf{Metric}: $\mathrm d_{\mathrm{geo}}$ (Poincaré/Lorentz)\\
       Exponential volume growth};

\node[mainbox=green!60!black,label={[font=\bfseries]above:Ultrametric}] (ultra) at (3,-3)
      {\textbf{Space}: $\mathbb Q_{p}$ or $\mathbb Z_{p}$\\
       \textbf{Metric}: $|x-y|_{p}$ (LCA depth)\\
       Strong triangle ineq.};

% --- Arrows & annotations ---------------------------------------
\draw[arr,blue!80!red] (eucl.east) -- 
    node[annotation,above,sloped,yshift=6pt,text width=3.2cm,pos=0.55]
    {Curvature $<0$\\plus Riem.\ opt.} (hyp.west);

\draw[arr,red!80!green!60!black] (hyp.south east) to[out=270,in=30] 
    node[annotation,right,pos=0.75,xshift=4pt]
    {Discretise \&\\take $p$-adic limit} (ultra.north east);

\draw[arr,green!60!black!80!blue] (ultra.west) to[out=180,in=270] 
    node[annotation,left,pos=0.7]
    {Ignore hierarchy\\$\Rightarrow$ high distortion} (eucl.south west);

% --- Method boxes ------------------------------------------------
\node[methodbox=blue] (mlp) at (-1.5,1.8) {MLP\\ResNet};
\node[methodbox=red] (poincare) at (7.5,1.8) {Poincaré\\Lorentz emb.};
\node[methodbox=green!60!black] (vpunn) at (3,-4.5) {\bfseries v-PuNN\\ (this work)};

\draw[arr,dashed,black!60] (mlp.east) -- (eucl.west);
\draw[arr,dashed,black!60] (poincare.west) -- (hyp.east);
\draw[arr,dashed,black!60] (vpunn.north) -- (ultra.south);

% --- Geometry labels above method boxes -------------------------
\node[font=\bfseries] at (mlp.north) [yshift=6pt] {Euclidean};
\node[font=\bfseries] at (poincare.north) [yshift=6pt] {Hyperbolic};

% --- Legend (moved further down) --------------------------------
\begin{scope}[shift={(0,-7.4)}]
\node[fill=white,fill opacity=0.9,text opacity=1,rounded corners=5pt,
      inner sep=8pt,text width=8.8cm,anchor=center] (legend) {
\begin{tabular}{@{}c@{\hspace{8pt}}l@{}}
\cellcolor{blue!20} & Euclidean space/models \\
\cellcolor{red!20}  & Hyperbolic space/models \\
\cellcolor{green!20}& Ultrametric space/models \\
\color{black!60}\tikz{\draw[dashed,->] (0,0) -- (0.8,0);} & Model associations \\
\end{tabular}
};
\node[font=\bfseries,above=3pt] at (legend.north) {Legend};
\end{scope}

% --- Overall title ----------------------------------------------
\node[font=\bfseries\large,anchor=south,yshift=10pt] 
      at (current bounding box.north) 
      {Geometric Spaces for Hierarchical Representation};

\end{tikzpicture}
\caption{Conceptual landscape of geometric spaces for hierarchical data embedding.
v-PuNNs leverage ultrametric spaces for exact tree geometry matching, while
Euclidean and hyperbolic approaches provide continuous approximations.
Dashed arrows indicate model families associated with each space.}
\label{fig:geometry_landscape}
\end{figure*}

\FloatBarrier

%====================================================================
\section{Mathematical Foundations: van der Put Neural Networks and
\Turl{}}
\label{sec:math_found}
%====================================================================
\noindent
Deep learning models routinely embed hierarchical data in Euclidean
vector spaces where small Euclidean distances do not reflect
taxonomy depth~\cite{linial1995geometry}.
Because Euclidean geometry satisfies only the
ordinary triangle inequality, it cannot natively encode the nested
``balls within-balls’’ structure of a rooted tree.  In contrast, the
$p$-adic integers $\Z_p$ form a non-Archimedean (ultrametric)
space in which that nesting appears by construction~\cite{gouvea1997p}.
This section builds a mathematically rigorous bridge from finite hierarchies to
$p$-adic analysis, culminating in the architecture of v-PuNNs governed by the
principle of TURL.

Throughout, let $p$ be a fixed prime and write
$\lvert\cdot\rvert_{p}$ for the $p$-adic norm.  For brevity, we define
\[
\operatorname{pref}_{k}(x)\;:=\;\sum_{i=0}^{k-1}a_i p^{\,i},
\]
the $k$-digit base-$p$ prefix of $x=\sum_{i=0}^{\infty}a_i p^{\,i}$.

%--------------------------------------------------------------------
\subsection{Mathematical Preliminaries and Notation}
\label{sec:prelim}
%--------------------------------------------------------------------

\begin{definition}[Hierarchical data]\label{def:hierarchy}
Throughout the paper, a hierarchy is a finite, rooted, directed
tree \(T\) in which every node has a unique parent except the root and
all information-bearing items (species, synsets, proteins, etc) appear as
distinct leaves.  We write \(K\coloneqq\depth(T)\) for the maximum root‑to‑leaf
distance and \(b_k\) for the branching factor at depth \(k\).
\end{definition}

\begin{remark}
This covers taxonomies, ontologies, organizational charts, directory
trees, and any data set endowed with a unique lowest common ancestor
structure.  The $p$‑adic formalism that follows assumes nothing beyond Definition \ref{def:hierarchy}.
\end{remark}

\subsubsection{\texorpdfstring{$p$-Adic}{p-Adic} integers and metric properties}
\label{sec:padic_ints}

Every $n\!\ge0$ admits the unique base-$p$ expansion
\(n=\sum_{i=0}^{\infty}a_i p^{\,i}\) with $a_i\!\in\!\{0,\dots,p-1\}$~\cite{robert2000course}.  % Added citation
The valuation and norm
\[
\nu_{p}(n)=\min\{\,i:a_i\neq0\},\qquad
\lVert n\rVert_{p}=p^{-\nu_{p}(n)}
\]
induce the ultrametric distance
$d_{p}(m,n)=\lVert m-n\rVert_{p}$.  For all
$x,y,z\!\in\!\Z_{p}$,
\[
d_{p}(x,z)\;\le\;\max\bigl\{d_{p}(x,y),\,d_{p}(y,z)\bigr\},
\]
the strong triangle inequality~\cite{gouvea1997p}.

\subsubsection{\texorpdfstring{$p$}{p}-Adic balls and tree structure}
\label{sec:padic_balls}

For $k\!\ge\!0$ and $a\!\in\!\Z_{p}$ define
\[
B_{k}(a)=\bigl\{x\in\Z_{p}:\nu_{p}(x-a)\ge k\bigr\}.~\cite{robert2000course}
\]
Each ball is clopen, $B_{k+1}(a)\subset B_{k}(a)$, and the set
$\{B_{k}(c)\}_{c\bmod p^{k}}$ partitions $\Z_{p}$.  Crucially,
\(x\in B_{k}(a)\Longleftrightarrow
\operatorname{pref}_{k}(x)=\operatorname{pref}_{k}(a)\):
digits label edges of a rooted $p$-ary tree~\cite{gouvea1997p}.

% ------------------------------------------------------------------
% p-adic ball and prefix path (depth-3 example, base p = 5)
% ------------------------------------------------------------------
\begin{figure}[H]
\centering
\begin{tikzpicture}[
  level distance=9mm,
  sibling distance=8.5mm,
  every node/.style={circle,draw,minimum size=5.5pt,inner sep=0pt},
  subtree/.style={draw=green!50,ultra thick},
  edge from parent/.style={draw=black,thin},
  edge from parent path={(\tikzparentnode) -- (\tikzchildnode)}
]

% -- root and full 5-ary tree depth 3 -----------------------------
\node (root) {}
  child { node {} }
  child { node (c1) {}
    child[subtree] { node[subtree] (c11) {}
      child[subtree] { node[subtree,fill=green!30] {} }
      child[subtree] { node[subtree,fill=green!30] {} }
      child[subtree] { node[subtree,fill=green!30] {} }
      child[subtree] { node[subtree,fill=green!30] {} }
      child[subtree] { node[subtree,fill=green!30] {} }
    }
    child { node {} }
    child { node {} }
    child { node {} }
    child { node {} }
  }
  child { node {} }
  child { node {} }
  child { node {} };

% -- digit labels along chosen path -------------------------------
\node[above left=0pt of root,font=\scriptsize] {$a_0$};
\node[left=1pt of c1,font=\scriptsize] {$a_1$};
\node[left=1pt of c11,font=\scriptsize] {$a_2$};

\end{tikzpicture}
\caption{A depth-\texorpdfstring{$3$}{3} subtree (shaded) forms a \texorpdfstring{$p$}{p}-adic ball
$B\!=B(a,p^{-3})$.  The digits $(a_0,a_1,a_2)$ are the successive sibling
indices along the root-leaf path and serve as the $p$-adic prefix
$\operatorname{pref}_3(a)$~\cite{gouvea1997p}.}
\label{fig:padic_ball_prefix}
\end{figure}
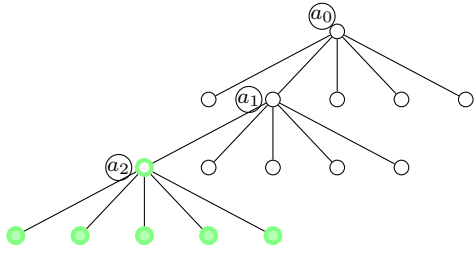

\subsubsection{van der Put--Inspired Hierarchical Indicator Family}
\label{sec:vdp_basis}

Set
\[
\chi_{B_{k}(c)}(x)=
\begin{cases}
1,&x\in B_{k}(c),\\
0,&x\notin B_{k}(c).
\end{cases}
\]
The family
\(
\mathcal V_{p}=\{\chi_{B_{k}(c)}:k\!\ge\!0,\,c\!\in\!\Z_{p}\}
\)
is a natural \emph{hierarchical spanning family} adapted to the nested,
clopen ball structure of \(\Z_p\).  It is \emph{not} a basis in the strict
functional-analytic sense: across depths it is linearly dependent, since
every ball partitions into \(p\) sub-balls and therefore
\[
\chi_{B_{k}(c)}(x)=\sum_{i=0}^{p-1}\chi_{B_{k+1}(c+i\,p^{k})}(x)
\qquad\forall\,x\in\Z_p.
\]
For our finite-depth setting, linear independence is unnecessary: we only
require an explicit, depth-aligned family whose elements correspond to
subtrees (prefix classes) and can be optimized depth-wise.

\begin{remark}[Relation to the classical van der Put basis]
The classical van der Put theory (1968) constructs a Schauder basis of
\(C(\Z_p,\Q_p)\) from \emph{differences} of ball indicators; the resulting
family is linearly independent.  Our use of raw indicators is inspired by
that theory but intentionally prioritizes explicit subtree semantics and
transparent parameter-to-subtree correspondence.
\end{remark}

Truncating to depth \(K\) uses
\(N=\sum_{j=0}^{K-1}p^{\,j}\) indicator functions, matching the parameter
budget of a depth-\(K\) v-PuNN in our finite-hierarchy instantiation.
\subsubsection{Prefix code for leaves}
\label{sec:label_encoding}

Let $\mathcal L=\{\ell_{1},\dots,\ell_{N}\}$ denote the leaves of a
depth-$K$ taxonomy.  Choosing
\(p\ge N^{1/K}\), assign each leaf
\[
z(\ell_i)=\sum_{j=0}^{K-1}a_{ij}\,p^{\,j},
\qquad a_{ij}\in\{0,\dots,p-1\}.~\cite{gouvea1997p}
\]
Prefixes correspond one-to-one with internal nodes.\footnote{%
WordNet nouns: $(p,K)=(409,19)$:digit $\operatorname{d}_{18}$ isolates
synsets, $\operatorname{d}_{7}$ clusters lexical files, and
$\operatorname{d}_{0}$ splits living/non-living~\cite{miller1995wordnet}.}

\paragraph{Digit numbering convention;}
Throughout the paper, we number $p$‑adic digits from root to leaf:
\[
d_{K-1}\;(\text{root})\,,\;\;d_{K-2},\dots ,d_{1},\;d_{0}\;(\text{leaf}).
\]
For WordNet ($K=19$) this means the coarsest digit is $d_{18}$ and the
finest is $d_{0}$.

\subsubsection{Ultrametric loss}
\label{sec:ultra_loss}

Given input $x$ with label $\ell$ and model output
$\hat z(x)\!\in\!\Z_{p}$, define
\[
\mathcal L(x,\ell)=
\sum_{k=0}^{K-1}\lambda_{k}\;
\mathbf 1\!\Bigl[
\operatorname{pref}_{K-k}\bigl(\hat z(x)\bigr)\neq
\operatorname{pref}_{K-k}\!\bigl(z(\ell)\bigr)
\Bigr],
\]
with weights $\lambda_{k}$ to emphasize coarse or fine levels. This choice of loss function is particularly powerful because it directly enforces the desired hierarchical structure during optimization. By the definition of the p-adic norm, a small distance $|f(x) - y|_p = p^{-k}$ for a large $k$ implies that the p-adic representations of the prediction $f(x)$ and the label $y$ are identical up to digit $k$. This mathematical property has a direct and intuitive structural interpretation~\cite{gouvea1997p}: minimizing the p-adic loss implicitly forces the model to learn the correct shared path from the root of the hierarchy, effectively placing the prediction in the same deep and specific subclade as the ground truth label. In practice (\ref{sec:vapo}) we implement $\mathcal L$ as a Huffman-weighted cross-entropy~\cite{huffman1952method} that is differentiable almost everywhere.

\paragraph{Notation summary}
\begin{center}
\begin{tabular}{@{}ll@{}}
\toprule
Symbol & Meaning\\
\midrule
$p$ & prime base of the code\\
$K$ & tree depth / digits per code\\
$z(\ell)$ & $p$-adic code of leaf $\ell$\\
$\operatorname{d}_{k}$ & digit at place $p^{\,k}$ ($k=0$: root)\\
$B_{k}(c)$ & ball of radius $p^{-k}$ centered at $c$\\
$\mathcal V_{p}$ & van der Put basis\\
\bottomrule
\end{tabular}
\end{center}

%--------------------------------------------------------------------
\subsection{Hierarchies as Ultrametric Spaces}
\label{sec:ultrametric_hier}
%--------------------------------------------------------------------

For leaves $x,y$ let
$k=\depth\bigl(\operatorname{LCA}(x,y)\bigr)\in\{0,\dots,K\!-\!1\}$.
Defining
\(d(x,y)=\exp(-\alpha\,k)\) with $\alpha>0$ gives
\[
d(x,z)\;\le\;\max\{d(x,y),d(y,z)\},
\]
so $(\mathcal L,d)$ is ultrametric~\cite{rammal1986ultrametricity}.  Hyperbolic embeddings~\citep{nickel2017poincare} approximate this geometry only locally, whereas
the $p$-adic integers realize it exactly.

%--------------------------------------------------------------------
\subsection{A Canonical \texorpdfstring{$p$-Adic}{p-Adic} Embedding}
\label{sec:canonical_embed}
%--------------------------------------------------------------------

\begin{lemma}[Isometry]
\label{lem:isometry}
Let $x,y$ be leaves and
\(k=\depth\bigl(\operatorname{LCA}(x,y)\bigr)\).
Then
\[
\bigl|f(x)-f(y)\bigr|_{p}=p^{-k}.
\]
\end{lemma}

\begin{proof}[Sketch]
$f$ encodes the sibling index at each depth as one base-$p$ digit~\cite{robert2000course}.
The first differing digit occurs at position $k$ (root is $0$), hence
$v_{p}(f(x)-f(y))=k$ and
$\lvert f(x)-f(y)\rvert_{p}=p^{-k}$.
\end{proof}

\paragraph{Toy example ($p\!=\!2,K\!=\!3$).}
\vspace{-1ex}
\begin{figure}[H]
\centering
\begin{tikzpicture}[level distance=8mm,
    every node/.style={circle,draw,inner sep=1pt},
    level 1/.style={sibling distance=20mm},
    level 2/.style={sibling distance=12mm}]
\node (r) {0}
  child { node {0}
    child { node[label=below:$\ell_{1}$] {0} }
    child { node[label=below:$\ell_{2}$] {1} }
  }
  child { node {1}
    child [missing]
    child { node[label=below:$\ell_{3}$] {3} }
  };
\end{tikzpicture}
\,
\begin{minipage}{0.45\linewidth}
\begin{align*}
f(\ell_{1})&=0\cdot2^{0}+0\cdot2^{1}+0\cdot2^{2}=0,\\
f(\ell_{2})&=1\cdot2^{0}+0\cdot2^{1}+0\cdot2^{2}=1,\\
f(\ell_{3})&=1\cdot2^{0}+1\cdot2^{1}+0\cdot2^{2}=3.
\end{align*}
Common prefix lengths:
$k_{12}=2$ ($p$-adic distance $2^{-2}$),
$k_{13}=1$ ($2^{-1}$),
$k_{23}=1$ ($2^{-1}$).
\end{minipage}
\vspace{-2ex}
\end{figure}

%--------------------------------------------------------------------
\subsection{The v-PuNN Neuron}
\label{sec:vpunn_neuron}
%--------------------------------------------------------------------

For any ball $B\!=\!B_{k}(a)$ define
\[
\chi_{B}(x)=\begin{cases}
1,&x\in B,\\
0,&x\notin B,
\end{cases}
\quad
\widetilde{\chi}_{B}(x)=\chi_{B}(x)+\alpha\bigl(1-\chi_{B}(x)\bigr)
\]
with $\alpha\!\approx\!10^{-2}$ to smooth the objective. Although the underlying indicator family is linearly dependent across depths, this poses no issue for learning because optimization is carried out over a fixed finite family and each sample activates a single nested branch of indicators. A depth-$K$ v-PuNN is
\[
F(x)=\sum_{i=1}^{N}w_{i}\,\widetilde{\chi}_{B_{i}}(x),\qquad
w_{i}\in\Q_{p},\;
N=\sum_{j=0}^{K-1}p^{\,j}.
\]

%--------------------------------------------------------------------
\subsection{Transparency and Interpretability (\Turl{})}
\label{sec:turl}
%--------------------------------------------------------------------

%--------------------------------------------------------------------
\begin{theorem}[Parameter-Subtree Duality]\label{thm:param-subtree}
For a depth-\(K\) \vpunn, each learnable coefficient
\(c_B\) is in one-to-one correspondence with a unique \(p\)-adic ball
\(B\subset\Z_p\) (equivalently, with a unique internal node of the
depth-\(K\) hierarchy).  Conversely, every subtree possesses exactly one
such coefficient, and no other parameter can affect that subtree.
\end{theorem}

\begin{proof}
Lemma~\ref{lem:isometry} establishes an isometry between the set of
leaves and the family \(\mathcal B_K\) of depth-\(K\) \(p\)-adic balls~\cite{gouvea1997p},
so each internal node is uniquely identified by a ball
\(B=\pref{K}{x}\).  A \vpunn\ expresses its output as
\[
  F(x)\;=\;\sum_{B\in\mathcal B_K} c_B\,\chi_{B}(x),
\]
hence every coefficient \(c_B\) multiplies the indicator of exactly
one ball, proving the forward direction.

For the converse, let \(B^\star\) be any depth-\(K\) node.  Its indicator
\(\chi_{B^\star}\) appears with coefficient \(c_{B^\star}\) in the
expansion above, so the entire subtree rooted at \(B^\star\) is governed
solely by \(c_{B^\star}\).

Finally, if \(x\notin B\) then \(\chi_{B}(x)=0\).  Therefore
\[
  \frac{\partial F(x)}{\partial c_B}=0
  \quad\text{and}\quad
  \frac{\partial\mathcal L}{\partial c_B}=0
  \qquad\forall\,x\notin B,
\]
so a coefficient \(c_B\) can influence the loss only through samples
inside its own subtree and never through any other branch.  This completes
the bijective correspondence.
\end{proof}
%--------------------------------------------------------------------

%--------------------------------------------------------------------
\subsection{From van der Put to Finite Hierarchical Approximation}
\label{sec:vdp_to_fha}
%--------------------------------------------------------------------

\paragraph{Classical roots.}
In 1968, the Dutch mathematician \textsc{Marius van der Put} introduced a
canonical expansion of continuous functions on \(\Z_p\) using functions
built from \(p\)-adic balls.  In its classical form, the construction
yields a \emph{Schauder basis} of \(C(\Z_p,\Q_p)\) formed from \emph{differences}
of characteristic functions of nested balls.  We echo this historical
origin in the name \vpunn\ (\emph{van der Put Neural Networks}), while
emphasizing that our architecture uses raw ball indicators as a
depth-aligned spanning family to obtain explicit subtree semantics.

\vspace{0.5\baselineskip}
\noindent

\begin{theorem}[van der Put, 1968 (classical form)]
\label{thm:vdp}
Fix a prime \(p\).  There exists an explicitly indexed family
\(\{e_n\}_{n\ge 0}\subset C(\Z_p,\Q_p)\) with \(e_0\equiv 1\) such that for
every \(n\ge 1\), \(e_n\) is a difference of characteristic functions of
two nested \(p\)-adic balls, and every continuous
\(f:\Z_p\to\Q_p\) admits a unique uniformly convergent expansion
\[
f(x)=\sum_{n=0}^{\infty}\beta_n\,e_n(x),
\qquad \beta_n\in\Q_p.
\]
\end{theorem}

\paragraph{Why a new theorem is needed?}
Theorem~\ref{thm:vdp} speaks to infinite $p$-adic space, but real
datasets live in finite, depth-\(K\) hierarchies.  In that
setting, every data point sits inside some ball of radius \(p^{-K}\); the
higher-depth balls never occur.  Hence, the infinite series is wildly
over-parameterized.  Closing this gap yields our main theoretical
contribution.

%---------------------------------------------------------------
\begin{proof}
Let \(h = g\circ f^{-1} : \Z_p \to \Q_p\).  Since the set
\(f(L(T)) \subset \Z_p\) is finite and every element is represented with
$p$-adic precision at most \(K\), the function \(h\) is constant on each
depth-\(K\) prefix class (equivalently, on each ball at the finest
resolution induced by the hierarchy).

Consequently, \(h\) lies in the span of the depth-aligned indicator family
\(\{\chi_{B}(x): B\in\mathcal B_K\}\): one may assign coefficients so that
\(F\) matches \(h\) on every \(x=f(\ell)\) by setting the coefficient of
the unique finest prefix class containing \(x\) to the constant value of
\(h\) on that class (and setting all remaining coefficients to zero).
Thus \(F(f(\ell))=g(\ell)\) for all \(\ell\in L(T)\), hence the
approximation holds (indeed, with equality) on the leaf set.

The bound
\[
N = \sum_{j=0}^{K-1} p^j
\]
counts the number of prefix classes (balls) available up to depth \(K\)
in our finite hierarchy instantiation, so the representation above uses at
most \(N\) such indicators.

Now, let \(\hat F(x) = \sum_{B \in \mathcal B_K} \hat c_B\,\chi_B(x)\)
be any perturbed version such that
\[
|c_B - \hat c_B|_p < \varepsilon / (p^K - 1)
\qquad \text{for all } B \in \mathcal B_K.
\]
Then, for all \(x \in f(L(T))\), at most \(K\) of the terms
\(\chi_B(x)\) are nonzero (since the balls are nested), so the total
difference satisfies:
\[
|F(x) - \hat F(x)|_p < K \cdot \varepsilon / (p^K - 1)
< \varepsilon,
\]
because \(K < p^K - 1\) for \(K \ge 1\), completing the proof.
\end{proof}

\paragraph{Implications.}
Theorem~\ref{thm:fha} shows that a depth-\(K\) v-PuNN is
universally expressive for any function on a \(K\)-level
hierarchy, with a parameter budget that grows only geometrically in
\(K\). In practice, this gives
\[
\begin{aligned}
&N\approx3.0\times10^{6} &&\text{(WordNet nouns)},\\
&N\approx1.9\times10^{6} &&\text{(Gene Ontology, molecular function)},\\
&N\approx2.1\times10^{6} &&\text{(NCBI Mammalia taxonomy).}
\end{aligned}
\]
Thus, v-PuNNs achieve theoretical completeness with a footprint orders
of magnitude smaller than would follow from the infinite series
in Theorem~\ref{thm:vdp}.

% ------------------------------------------------------------------
\paragraph{Theory vs.\ instantiated parameter counts.}
\label{sec:param_bridge}
Equation~\eqref{eq:vdp_param_bound} (Theorem~\ref{thm:fha}) gives the
functional worst-case parameter budget for a depth-$K$ v-PuNN:
\begin{equation}
N_{\text{vdp}}(K,p)
  = \sum_{j=0}^{K-1} p^{\,j}
  = \frac{p^{K}-1}{p-1}.
\label{eq:vdp_param_bound}
\end{equation}
This counts one coefficient for every $p$-adic ball at all depths
up to $K{-}1$ and therefore scales geometrically in $K$.

Our concrete \HiPaN{} implementation does not instantiate this
full van der Put tensor.  Instead, we factor the hierarchy
digit-wise: at each depth we learn a conditional head that
maps a parent digit to a child digit.  Because there are $p$ possible
parent digits, each head stores $O(p)$ (root) or $O(p^2)$ (conditional)
parameters, giving a total
\begin{equation}
N_{\text{HiPaN}}
  = p                       % root scalar layer
  + p^{2}                   % depth-1 dense conditional
  + (K_{\text{heads}}-1)(p^{2}+p) % remaining conditional heads
  = O(K_{\text{heads}}p^{2}).
\label{eq:hipan_param_formula}
\end{equation}

\paragraph{Expressiveness caveat.}
The conditional‑head factorization of Eq.~\eqref{eq:hipan_param_formula}
is a design choice, not a theorem: it can realize all functions proved
possible by Theorem~\ref{thm:fha} \textbf{only if} the hierarchy is digit‑separable.
On real corpora (WordNet, GO, NCBI) this empirically holds
(\S\ref{sec:experiments}), yet a contrived adversarial tree could break the
assumption. In that case, the full van der Put tensor
($N_{\text{vdp}}$ parameters) would be required for universality.

Crucially, the exponential combinatorics of $N_{\text{vdp}}$ are
implicit in the routing induced by the digit predictions: each
input activates exactly one branch per depth, so all $p^{j}$ balls at
depth $j$ are representable without allocating separate coefficients.

For the WordNet noun hierarchy we use $p=409$ and
$K_{\text{heads}}=18$ learnable depths, yielding
$N_{\text{HiPaN}}=3{,}018{,}420$ parameters
(\S\ref{sec:hipan_complexity}).  Evaluation prints an additional digit
$d_{18}$ for full 19-level paths, but that digit is weight-tied
and adds no parameters.

\vspace{0.5\baselineskip}
\noindent
\textbf{Key findings.}  v-PuNNs constitute a complete, geometrically
faithful, and white-box architecture for hierarchical data.
These mathematical foundations motivate the optimization strategy
presented in §\ref{sec:vapo}.

%====================================================================
\section{\Vpunn \ Architecture and Transparent Ultrametric Representation Learning}
\label{sec:vpunn}
%====================================================================

% ------------------------------------------------------------------
% Figure: end-to-end v-PuNN / HiPaN diagram
% ------------------------------------------------------------------
\begin{figure*}[!htbp]
\centering
\begin{tikzpicture}[>=stealth,thick,
    neuron/.style   ={circle,draw,minimum size=8pt,inner sep=0pt},
    atom/.style     ={rectangle,draw,minimum width=10pt,minimum height=8pt,
                      fill=blue!10},
    coeff/.style    ={circle,draw,minimum size=7pt,inner sep=0pt,fill=green!10},
    head/.style     ={rectangle,draw,rounded corners,minimum width=12pt,
                      minimum height=10pt,fill=orange!15},
    sumnode/.style  ={circle,draw,minimum size=10pt,fill=red!20},
    opnode/.style   ={rectangle,draw,minimum size=6pt,fill=gray!20},
    myarrow/.style  ={->,>=stealth,line width=.8pt},
    scale=1.0]

% --- LEFT PANEL: p-adic pipeline with reconstruction --------------
\node[neuron,label=left:{$x \in \mathbb{Z}_{p}$}] (x) at (0,0) {};

% Atoms and coefficients
\foreach \i/\y/\d in {0/-1.8/{d_0},1/0/{d_1},2/1.8/{d_2}} {
  \node[atom,label=above:{\scriptsize$\chi_{B_{\i}}$}] (a\i) at (2,\y) {};
  \draw[myarrow] (x) -- (a\i);
  
  \node[coeff,label={[yshift=6pt]right:{\scriptsize$c_{B_{\i}} \in \mathbb{Q}_{p}$}}] (c\i) at (4,\y) {};
  \draw[myarrow] (a\i) -- (c\i);
  
  % Heads with depth labels
  \node[head,label=right:{\scriptsize Head (d)}] (h\i) at (6,\y) {};
  \draw[myarrow] (c\i) -- (h\i);
}

% --- PREDICTION SYNTHESIS (NEW ADDITION) -------------------------
\node[below right=0.5cm and 0.2cm of h0, anchor=west] (d0) {\scriptsize $d_0$};
\node[above right=0.5cm and 0.2cm of h1, anchor=west] (d1) {\scriptsize $d_1$};
\node[above right=0.8cm and 0.2cm of h2, anchor=west] (d2) {\scriptsize $d_2$};

\draw[myarrow, dashed] (h0.east) -- (d0);
\draw[myarrow, dashed] (h1.east) -- (d1);
\draw[myarrow, dashed] (h2.east) -- (d2);

% p-adic reconstruction operators
\node[opnode] (p0) at (8,-1.5) {$p^0$};
\node[opnode] (p1) at (8,0) {$p^1$};
\node[opnode] (p2) at (8,1.5) {$p^2$};

% Multiplication nodes
\node[opnode] (mul0) at (9.5,-1.5) {$\times$};
\node[opnode] (mul1) at (9.5,0) {$\times$};
\node[opnode] (mul2) at (9.5,1.5) {$\times$};

% Summation nodes
\node[sumnode] (sum0) at (11,-0.5) {$+$};
\node[sumnode] (sum1) at (11,0.5) {$+$};
\node[neuron] (output) at (12.5,0) {$\hat{x}$};

% Connections
\draw[myarrow] (d0) -| (p0);
\draw[myarrow] (d1) -| (p1);
\draw[myarrow] (d2) -| (p2);

\draw[myarrow] (p0) -- node[midway,above] {$d_0$} node[midway,below=3pt] {\scriptsize $1$}   (mul0);
\draw[myarrow] (p1) -- node[midway,above] {$d_1$} node[midway,below=3pt] {\scriptsize $p$}   (mul1);
\draw[myarrow] (p2) -- node[midway,above] {$d_2$} node[midway,below=3pt] {\scriptsize $p^2$} (mul2);

\draw[myarrow] (mul0) -- (sum0);
\draw[myarrow] (mul1) -- (sum0);
\draw[myarrow] (mul2) -- (sum1);
\draw[myarrow] (sum0) -- (sum1);
\draw[myarrow] (sum1) -- (output);

% Reconstruction formula
\node[below=0.9cm of output] (formula) {\scriptsize $\hat{x} = d_0 + d_1 \cdot p + d_2 \cdot p^2$};

% --- RIGHT PANEL: van der Put basis visualization (now below) ----
\begin{scope}[yshift=-6.5cm,xshift=0cm,yscale=0.8]  % << moved below the main panel
  % Axes
  \draw[->] (-.2,0) -- (4.3,0) node[right]{$x \in \mathbb{Z}_{p}$};
  \draw[->] (0,-.2) -- (0,2.6) node[above]{$\varphi(x)$};

  % Target signal (p-adic step function)
  \draw[blue,line width=1.2pt] (0,0.3) -- (0.8,0.3) -- (0.8,1.0) -- (1.6,1.0) 
        -- (1.6,0.4) -- (2.4,0.4) -- (2.4,1.8) -- (3.2,1.8) -- (3.2,0.2) -- (4.0,0.2);
  
  % p-adic rectangles (balls)
  \foreach \x/\h/\c in {0.0/0.3/red!50, 0.8/1.0/green!50, 1.6/0.4/blue!50, 2.4/1.8/orange!50, 3.2/0.2/purple!50} {
    \draw[\c,fill=\c,fill opacity=0.3,opacity=.7] (\x,0) rectangle (\x+0.8,\h);
    \draw[myarrow,dashed] (\x+0.4,\h+0.1) -- ++(0,0.3);
  }

  % Series expansion formula
  \node at (2,2.3) {\small$\displaystyle
    \varphi(x) = \sum_{k=0}^{\infty} c_{B_k} \cdot \chi_{B_k}(x)$};
\end{scope}

% --- LEGEND -------------------------------------------------------
\begin{scope}[yshift=-12.5cm,xshift=0cm]
  \matrix [draw,fill=white,rounded corners,column sep=5mm,row sep=3mm] {
    \draw[myarrow] (0,0) -- (1,0); & \node[anchor=west] {\footnotesize Data flow}; \\
    \node[atom] {}; & \node[anchor=west] {\footnotesize Characteristic function $\chi_B$}; \\
    \node[coeff] {}; & \node[anchor=west] {\footnotesize Coefficient $c_B$}; \\
    \node[head] {}; & \node[anchor=west] {\footnotesize Digit prediction head}; \\
    \draw[black,line width=1.2pt] (0,0) -- (1,0); & \node[anchor=west] {\footnotesize Target signal $\varphi(x)$}; \\
    %\draw[red!50,fill=red!30,opacity=.7] (0,-0.1) rectangle (0.6,0.1); & \node[anchor=west] {\footnotesize p-adic ball $B_k$}; \\
    \node[sumnode] {}; & \node[anchor=west] {\footnotesize Summation}; \\
    \node[opnode] {}; & \node[anchor=west] {\footnotesize Operation}; \\
  };
\end{scope}

\end{tikzpicture}
\caption{\textbf{HiPaN architecture.}
Input $x \in \mathbb{Z}_{p}$ activates characteristic functions of p-adic balls $B_k$, scaled by coefficients $c_{B_k} \in \mathbb{Q}_{p}$. Each coefficient feeds a specialized prediction head for a p-adic digit $d_k$. The digit outputs are combined through p-adic reconstruction: $\hat{x} = \sum_k d_k \cdot p^k$. 
Below: A depth-1 van der Put basis illustration (\(p=5\)).  Each colored outline is the indicator of a radius-\(p^{-1}\) ball; their weighted sum forms a piece-wise constant function.}
\label{fig:hipan_full_diagram}
\end{figure*}

\FloatBarrier

\noindent
A v-PuNN implements, in finite form, the infinite expansion of Theorem \ref{thm:vdp}.  Two design principles govern the construction:

\begin{enumerate}[label=(\roman*),nosep,leftmargin=2em]
\item \textbf{Transparent Ultrametric Representation Learning
      (\Turl).}\;
      Every intermediate representation must remain in
      $\Z_{p}$, and every trainable quantity must correspond to a unique
      $p$-adic ball.
\item \textbf{Finite Hierarchical Completeness.}\;
      By the Finite Hierarchical Approximation
      Theorem~\ref{thm:fha} a depth-\(K\) v-PuNN with
      \[
        N=\sum_{j=0}^{K-1}p^{\,j}
      \]
      coefficients is already universally expressive for any
      \(K\)-level hierarchy.  No additional parameters are required.
\end{enumerate}

%--------------------------------------------------------------------
\subsection{Neuron type: characteristic balls}
\label{sec:vpunn_neuron_arch}
%--------------------------------------------------------------------

\begin{definition}[Characteristic-ball neuron]
For a ball \(B=B(a,p^{-D})\subset\Z_{p}\) define
\[
\chi_{B}(x)\;=\;
\begin{cases}
1 & x\in B,\\
0 & x\notin B.
\end{cases}
\]
To provide a finite-difference signal during optimization we use the
leaky indicator
\(
\widetilde\chi_{B}(x)=\chi_{B}(x)+\alpha(1-\chi_{B}(x)),
\)
with \(\alpha=0.01\) throughout the paper.  Because
\(\chi_{B}\chi_{B'}=\chi_{B\cap B'}\), no additional non-linearity is
required.
\end{definition}

%--------------------------------------------------------------------
\subsection{Single-depth operator}
\label{sec:vpunn_layer}
%--------------------------------------------------------------------

Let \(\mathcal B_{D}=\{B_{1},\dots,B_{m}\}\) be the collection of all
balls of radius \(p^{-D}\) selected for a given layer.

\begin{definition}[van der Put layer]
\[
  \boxed{\;
    \Phi_{D}:\Z_{p}\longrightarrow\Q_{p}^{m},\quad
    (\Phi_{D}x)_{i}=c_{B_{i}}\,
    \widetilde{\chi}_{B_{i}}(x)
  \;}
\]
where each coefficient is stored as
\(
c_{B}=p^{\,v_{B}}\;u_{B},\;
v_{B}\in\Z,\;
u_{B}\in\{1,\dots,p-1\}.
\)
\end{definition}

\paragraph{Lipschitz property.}
If \(d_{\nu}(x,y)=p^{-\!\min\{k:x_k\neq y_k\}}\) is the valuation metric,
then \(\Phi_{D}\) is \(1\)-Lipschitz:
\(
\lVert\Phi_{D}(x)-\Phi_{D}(y)\rVert_{\infty}\le p^{-\!D}
\)
because any change of value requires leaving a depth-\(D\) ball, a consequence of the non-Archimedean triangle inequality~\cite{gouvea1997p}.

\paragraph{Coefficient storage.}
A single scalar object (AdamScalar or GISTScalar)
stores \(c_{B}\) as a real number \(s\).  Rounding \(\operatorname{round}(s)
\bmod p\) returns the current digit; the integer part encodes the
valuation \(v_{B}\).  Updating \(s\) by at most \(\pm1\) therefore
modifies exactly one $p$-adic digit, satisfying \Turl.

%--------------------------------------------------------------------
\subsection{Network construction}
%--------------------------------------------------------------------

Choose a (possibly sparse) depth schedule
\(0=D_{0}<D_{1}<\dots<D_{L-1}\le K-1\).  
The complete mapping is
\[
  x
  \;\xrightarrow{\;\Phi_{D_{0}}\;}
  z_{0}
  \;\xrightarrow{\;\Phi_{D_{1}}\;}
  z_{1}
  \;\dots\;
  \xrightarrow{\;\Phi_{D_{L-1}}\;}
  z_{L-1}
  \xrightarrow{\text{Digit heads}}
  \hat{\mathbf d}.
\]
Exactly one atom fires per depth, so inference costs \(O(L)\).

\paragraph{Depth-pruning policy.}
In practice we include depth \(D\) if the empirical KL-divergence between
digit distributions at depths \(D-1\) and \(D\) exceeds a
threshold~\(\varepsilon=10^{-3}\).  For WordNet this yields
\(D\in\{0,1,2,4,8,16\}\); the public artifact keeps the full
\(K=19\) layers for maximal transparency.

%--------------------------------------------------------------------
\subsection{Digit-prediction heads}
\label{sec:vpunn_heads}
%--------------------------------------------------------------------

\begin{itemize}[leftmargin=2.2em,nosep]
\item \textbf{Depth 0 (root).}  
      One scalar head per root digit (class \texttt{AdamScalar})
      produces a soft-max over \(\{0,\dots,p-1\}\).
\item \textbf{Depth 1.}  
      A dense mean-squared-error head
      (\texttt{DenseMSEHead}) regresses the child digit
      conditioned on its parent.
\item \textbf{Depths \(\ge2\).}  
      Huffman-weighted two-logit heads
      (\texttt{TwoLogitCEHead}) implement the hierarchy-aware loss of
      §\ref{ssec:vapo_family}.
\end{itemize}

%--------------------------------------------------------------------
\subsection{Prime selection and parameter count}
\label{sec:vpunn_prime_count}
%--------------------------------------------------------------------

Choose the smallest prime
\(p\ge b_{\max}+1\), where \(b_{\max}\) is the maximum branching factor
in the data.  Edge case: when \(b_{\max}=1\) we set \(p=2\); the theory degenerates smoothly to a binary lattice.

\[
N
  =\sum_{j=0}^{K-1}p^{\,j}
  =\frac{p^{K}-1}{p-1}
  \quad\text{parameters.}
\]
For WordNet (18 learnable heads, $p=409$) this is $N=3\,018\,420$.

%--------------------------------------------------------------------
\subsection{Transparency guarantee}
%--------------------------------------------------------------------

\begin{lemma}[Activation = ancestor chain]
\label{lem:transparency}
For every input \(x\in\Z_{p}\) the non-zero activations across all
depths coincide with the ancestor chain of \(x\) in the hierarchy.
\end{lemma}

\begin{proof}
\(\widetilde\chi_{B}(x)>0\) iff \(x\in B\).
The nested balls \(\{B(a,p^{-D})\}_{D=0}^{K-1}\) are exactly the
subtrees encountered along the root-to-leaf path of \(x\); no other
balls contain \(x\).
\end{proof}

%--------------------------------------------------------------------
\subsection{Visual intuition}
%--------------------------------------------------------------------

Figure~\ref{fig:hipan_full_diagram} illustrates sparsity for a single depth (Theorem~\ref{thm:fha}), while Figure~\ref{fig:hipan_vs_mlp} contrasts v-PuNN with a standard MLP fixed activations vs.\ learnable $p$-adic coefficients, Euclidean vs.\ valuation optimization, post-hoc vs.\ native interpretability.

%--------------------------------------------------------------------
\subsection*{Practical notes}
\begin{itemize}[leftmargin=2.2em,nosep]
\item \textbf{Leak parameter.}  We keep \(\alpha=0.01\); lowering below
      \(0.005\) stalls VAPO, raising above \(0.02\) blurs indicators.
\item \textbf{Parameter sharing.}  We do not share coefficients
      across siblings; each ball has its own weight to preserve strict
      subtree attribution required by \Turl.
\item \textbf{Binary trees.}  For arity-1 hierarchies (\(p=2\)) the
      two-logit heads reduce to single Bernoulli logits; all proofs hold
      verbatim.
\end{itemize}

\vspace{0.5\baselineskip}
\noindent
\textbf{Summary.}  
A v-PuNN layer is a 1-Lipschitz valuation-space operator whose weights
carry exact tree semantics; stacking such layers under \Turl\ yields a
model that is simultaneously complete, interpretable, and computationally efficient.

% ------------------------------------------------------------------
%  Figure: MLP vs HiPaN (v-PuNN) 
% ------------------------------------------------------------------
\begin{figure*}[!htbp]
\centering
\setlength{\arrayrulewidth}{0.7pt}
\renewcommand{\arraystretch}{1.12}

% Convenience macro for centred column of fixed width
\newcommand{\thcell}[1]{%
  \multicolumn{1}{|>{\centering\arraybackslash}m{0.46\linewidth}|}{#1}}

\begin{tabular}{|m{0.46\linewidth}|m{0.46\linewidth}|}
\hline
\multicolumn{2}{|c|}{\bfseries Conventional MLP \emph{vs.} HiPaN (v-PuNN instance)}\\\hline

% ------------------------------------------------------------------
% Theorem row
% ------------------------------------------------------------------
\thcell{\textbf{Universal Approximation Theorem}} &
\thcell{\textbf{Finite Hierarchical Approximation Theorem}} \\ \hline

% ------------------------------------------------------------------
% Formula row
% ------------------------------------------------------------------
\thcell{$
  f(\mathbf x)\approx
    \sum_{i=1}^{N(\varepsilon)}
      a_i\,\sigma\!\bigl(\boldsymbol w_i^{\!\top}\mathbf x+b_i\bigr)
$} &
\thcell{$
  g(x)=\sum_{B\in\mathcal B_D} c_B\,\chi_{B}(x)
$} \\ \hline

% ------------------------------------------------------------------
% Schematic row
% ------------------------------------------------------------------
\thcell{%
% ===== MLP TikZ ===================================================
\begin{tikzpicture}[thick,scale=0.78,>=latex]
  % input layer ----------------------------------------------------
  \foreach \i in {1,...,3}
    \node[circle,draw,fill=black!12,inner sep=2pt] (in\i) at (0,-\i*0.6) {};

  % hidden layer 1 -------------------------------------------------
  \foreach \i in {1,...,5}
    \node[circle,draw,inner sep=2pt] (hA\i) at (1.4,1.2-\i*0.6) {};

  % hidden layer 2 -------------------------------------------------
  \foreach \i in {1,...,4}
    \node[circle,draw,inner sep=2pt] (hB\i) at (2.8,0.9-\i*0.6) {};

  % output node ----------------------------------------------------
  \node[circle,draw,double,inner sep=2pt] (out) at (4.2,-0.6) {};

  % edges: in -> hA ------------------------------------------------
  \foreach \i in {1,...,3}
    \foreach \j in {1,...,5}
      \draw[blue!40] (in\i) -- (hA\j);

  % edges: hA -> hB ------------------------------------------------
  \foreach \i in {1,...,5}
    \foreach \j in {1,...,4}
      \draw[blue!40] (hA\i) -- (hB\j);

  % edges: hB -> out ----------------------------------------------
  \foreach \i in {1,...,4}
    \draw[red!45] (hB\i) -- (out);

  % labels ---------------------------------------------------------
  \node[font=\scriptsize,below left]  at (in3.south west) {$\mathbf x$};
  \node[font=\scriptsize,above]       at (out.north)     {$\hat y$};
  \node[font=\scriptsize,gray] at (0.7,0.0)  {W$_1$};
  \node[font=\scriptsize,gray] at (2.1,0.0)  {W$_2$};
  \node[font=\scriptsize,gray] at (3.5,0.0)  {W$_3$};
\end{tikzpicture}

\smallskip
\centering
\textit{fixed activations; dense real weights}}%
&
\thcell{%
% ===== HiPaN TikZ ================================================
\begin{tikzpicture}[thick,scale=0.78,>=latex]
  % input ----------------------------------------------------------
  \node[circle,draw,fill=black!12,inner sep=2pt] (xin) at (0,0) {};

  % atoms ----------------------------------------------------------
  \foreach \i/\y in {1/1.0,2/0,3/-1.0}{
    \node[rectangle,draw,fill=blue!15,minimum width=7pt,minimum height=7pt]
          (a\i) at (1.4,\y) {};
    \draw[->,blue!40] (xin) -- (a\i);
  }

  % coefficients ---------------------------------------------------
  \foreach \i/\y in {1/1.0,2/0,3/-1.0}{
    \node[circle,draw,fill=green!18,minimum size=6pt] (c\i) at (2.6,\y) {};
    \draw[->,green!50!black] (a\i) -- (c\i);
  }

  % heads ----------------------------------------------------------
  \foreach \i/\y/\d in {1/1.0/{$d_2$},2/0/{$d_1$},3/-1.0/{$d_0$}}{
    \node[rectangle,draw,rounded corners,fill=orange!28,
          minimum width=14pt,minimum height=8pt] (h\i) at (4.2,\y) {};
    \draw[->,gray!60] (c\i) -- (h\i);
    \node[font=\scriptsize,right=2pt of h\i] {\d};
  }

  % labels ---------------------------------------------------------
  \node[font=\scriptsize,below left] at (xin.south west) {$x\!\in\!\mathbb Z_p$};
\end{tikzpicture}

\smallskip
\centering
\textit{learnable $p$-adic coefficients $c_B$; routed by depth}}\\\hline

% ------------------------------------------------------------------
% Deep-formula row
% ------------------------------------------------------------------
\thcell{$
  \mathrm{MLP}(\mathbf x)=
    \sigma_3\!\circ W_3\circ
    \sigma_2\!\circ W_2\circ
    \sigma_1\!\circ W_1(\mathbf x)
$} &
\thcell{$
  \mathrm{HiPaN}(x)=
    \Phi_D\!\circ\dots\circ\Phi_1(x)
$}\\\hline

% ------------------------------------------------------------------
% optimization / interpretability row (Unicode-safe)
% ------------------------------------------------------------------
\thcell{\textit{Adam / SGD \quad\textbullet\quad post-hoc saliency (SHAP/LIME)}} &
\thcell{\textbf{VAPO} (valuation-adaptive) \quad\textbullet\quad native subtree attribution}\\\hline
\end{tabular}

% ------------------------------------------------------------------
% Legend strip
% ------------------------------------------------------------------
\vspace{0.7em}
\begin{tikzpicture}[baseline]
  \matrix[matrix of nodes,row sep=0.4em,column sep=1em]
  {
    \node[draw,fill=blue!10,minimum width=1.3em,minimum height=0.8em] {}; & data flow &
    \node[draw,fill=orange!28,minimum width=1.3em,minimum height=0.8em] {}; & digit head &
    \node[draw,fill=green!10!,minimum width=1.3em,minimum height=0.8em] {}; & $p$-adic coeff.\ $c_B$ \\ };
\end{tikzpicture}

\caption{Dense, opaque weights (\textbf{left}) versus sparse, structurally grounded
         $p$-adic atoms (\textbf{right}).  HiPaN replaces real weight edges with characteristic
         functions, learns coefficients in $\mathbb Q_p$, and uses a valuation-aware optimizer, yielding exact subtree attribution.}
\label{fig:hipan_vs_mlp}
\end{figure*}
\FloatBarrier

%====================================================================
\subsection{HiPaN Architecture and Workflow}
\label{sec:hipan_arch}
%====================================================================

HiPaN realizes a transparent hierarchical classifier in three
phases: \textit{input encoding}, \textit{hierarchical prediction}, and
\textit{output reconstruction}, while strictly preserving the
ultrametric geometry guaranteed by the van der Put neural network (vPuNN) architecture.  Every trainable quantity corresponds bijectively to a $p$‑adic subtree, enabling exact attribution.%
\footnote{HiPaN uses the leaky indicator
$\widetilde\chi_B(x)=\chi_B(x)+\alpha(1-\chi_B(x))$ with
$\alpha=0.01$; see §\ref{sec:vpunn_neuron}.  Empirically,
$\alpha<0.005$ stalls GIST-VAPO, whereas $\alpha>0.02$ degrades leaf
accuracy by blurring indicators.}

% ---------------------------------------------------------------
\subsubsection*{Notation}
\vspace{-0.4\baselineskip}
\begin{table}[H]
\centering
\begin{tabular}{ll}
\toprule
$K$ & maximum depth of the hierarchy \\
$p$ & prime $\ge B_{\max}+1$ (§\ref{sec:input_rep}) \\
$c_k$ & sibling index at depth $k$ \\
$\theta_k,\ d_k$ & true / predicted $p$‑adic digit at depth $k$ \\
$\Gamma,\Gamma^{-1}$ & path $\leftrightarrow$ integer isomorphism \\
$\Phi_k$ & van der Put layer at depth $k$ \\
$\mathcal H_k$ & digit‑prediction head at depth $k$ \\
\bottomrule
\end{tabular}
\end{table}

% ---------------------------------------------------------------
\subsubsection{Input Representation}
\label{sec:input_rep}
% ---------------------------------------------------------------

\begin{definition}[Hierarchy encoding]\label{def:path_encoding}
Let $\mathcal T$ be a rooted tree of maximum depth $K$.  
Each leaf $x$ has a unique root‑to‑leaf path
\[
\text{path}(x)=(c_{K-1},c_{K-2},\dots,c_0),
\quad
c_k\in\{0,1,\dots,b_k-1\},
\]
where $b_k$ is the branching factor at depth $k$.
\end{definition}

\paragraph{Prime selection.}
\[
p=\operatorname{next\_prime}(B_{\max}+1),
\qquad
B_{\max}=\max_k b_k.
\]

\begin{algorithm}[H]
\caption{Tree Construction (implementation)}
\label{alg:build_tree}
\begin{algorithmic}[1]
\Function{build\_tree}{node}
  \State $\text{kids}\gets\text{sorted\_hyponyms(node)}$
  \State $\text{children[node]}\gets\text{kids}$
  \If{$\text{kids}=\emptyset$}
      \State $\text{depth}\gets0,\;n\gets\text{node}$
      \While{$n\neq\text{root}$}
          \State $n\gets\text{parent\_of}[n];\;\text{depth}{+}{=}1$
      \EndWhile
      \State $\text{all\_leaves.append(node)}$
      \State $\text{leaf\_depths[node]}\gets\text{depth}$
  \EndIf
  \For{\textbf{each} $(\text{index},\text{child})\in\text{enumerate(kids)}$}
      \State $\text{parent\_of[child]}\gets\text{node}$
      \State $\text{sibling\_index[child]}\gets\text{index}$
      \State \Call{build\_tree}{child}
  \EndFor
\EndFunction
\end{algorithmic}
\end{algorithm}

\paragraph{Path $\to$ integer bijection.}
\[
\Gamma(x)=\sum_{k=0}^{K-1} c_k\,p^k\in\Z/p^{\,K}\Z.
\]

\begin{algorithm}[H]
\caption{Path Encoding (implementation)}
\label{alg:encode_path}
\begin{algorithmic}[1]
\Function{encode\_path}{synset}
  \State $\text{digits}\gets[\,];\;n\gets\text{synset}$
  \While{$n\neq\text{root}$}
      \State $\text{digits.append}(\text{sibling\_index}[n])$
      \State $n\gets\text{parent\_of}[n]$
  \EndWhile
  \State \text{Pad digits with 0 to length $K$}
  \State \Return $\displaystyle \sum_{k=0}^{K-1}\text{digits}[k]\;p^k$
\EndFunction
\end{algorithmic}
\end{algorithm}

% ---------------------------------------------------------------
\subsubsection{Hierarchical Prediction}
\label{sec:hier_prediction}
% ---------------------------------------------------------------

Digits are predicted root‑to‑leaf using depth‑specialized heads,
optimized by VAPO.

\begin{enumerate}[label=(\arabic*),leftmargin=2em]
\item \textbf{Root digit $k\!=\!K-1$:}
      \[
      \hat d_{K-1}=\arg\max_{i\in[0,p-1]}\theta^{(0)}_{i},
      \]
      via an AdamScalar soft‑max.
\item \textbf{Depth‑1 digit $k\!=\!K-2$:}
      \[
      \hat d_{K-2}
      =\bigl\lceil\theta^{(1)}_{\hat d_{K-1}}\bigr\rfloor_{p},
      \]
      predicted by DenseMSEHead.
\item \textbf{Deep digits $k\le K-3$:}
      with logits $-\tau(v-t)^2$ (TwoLogitCEHead, $\tau=0.5$),
      \[
      \hat d_k=
      \begin{cases}
      t & \text{if good logit $>$ other logit},\\
      \lfloor v_{\text{other}}\rceil_{p} & \text{otherwise}.
      \end{cases}
      \]
\end{enumerate}

\begin{theorem}[Sparse activation]\label{thm:ancestor_activation}
Exactly one head fires per depth; inference therefore costs $O(K)$.
\end{theorem}

% ---------------------------------------------------------------
\subsubsection{Output Reconstruction}
\label{sec:output_recon}
% ---------------------------------------------------------------

Decoding reverses $\Gamma$:
\[
\hat{\text{path}}=(\hat d_{K-1},\dots,\hat d_0).
\]

\begin{algorithm}[H]
\caption{Decoding (implementation)}
\label{alg:decode_path}
\begin{algorithmic}[1]
\Function{decode}{\,$\hat{\text{digits}}$}
  \State $current\gets\text{root};\;path\gets[\,]$
  \For{$k\gets K-1$ \textbf{downto} $0$}
      \State $d\gets\hat{\text{digits}}[k]$
      \State $current\gets\text{children}(current)[d]$
      \State $path.\text{append}(current)$
  \EndFor
  \State \Return leaf \textit{synset} in $path$
\EndFunction
\end{algorithmic}
\end{algorithm}

% ---------------------------------------------------------------
\subsubsection{Training Curriculum}
\label{sec:training_process}
% ---------------------------------------------------------------

\begin{table}[H]
\centering
\caption{Depth‑aware training schedule}\label{tab:hipan_training}
\begin{tabular}{lccc}
\toprule
\textbf{Phase} & \textbf{Epochs} & \textbf{LR} & \textbf{Active digits} \\
\midrule
Deep‑head warm‑up & 8   & 0.03 & $k\ge 2$ \\
Root warm‑up      & 4   & 0.03 & $k\in\{K-1,K-2\}$ \\
Fine‑tuning       & 100 & 0.015& all $k$ \\
\bottomrule
\end{tabular}
\end{table}

\paragraph{Key techniques.}
\begin{itemize}[leftmargin=1.9em]
\item \textbf{Digit‑wise shuffling} each epoch.
\item \textbf{Huffman weighting}\;
      $w_{(p,c)}=1/\sqrt{\text{count}(p,c)}$ for $k\ge 2$.
\item \textbf{Checkpointing} every 20 epochs.
\end{itemize}

\vspace{0.2\baselineskip}
The optimization schedule is backed by the convergence guarantees of
VAPO (Corollary~\ref{cor:adam_conv} for Adam‑VAPO and
Proposition~\ref{prop:sgist} for stochastic GIST-VAPO), ensuring depth‑wise
stationarity within the allotted epochs.

% ---------------------------------------------------------------
\subsubsection{Interpretability Primitives}
\label{sec:interpretability}
% ---------------------------------------------------------------

\begin{enumerate}[leftmargin=1.8em]
\item \textbf{Ball inspection}\;
      $\texttt{describe\_ball}(\theta,k)\!\to$
      synsets, gloss tokens, lexical stats.
\item \textbf{JSON export}\;
      $\texttt{export\_tree\_for\_viz}()$.
\item \textbf{Activation‑path visualization}\;
      non‑zero activations = ancestor chain
      (Theorem.~\ref{thm:ancestor_activation}).
\end{enumerate}

% ---------------------------------------------------------------
\subsubsection{Mathematical Formulation}
\label{sec:hipan_math}
% ---------------------------------------------------------------

\[
f:\Z/p^{\,K}\Z \to \mathcal Y,
\quad
f(x)=\Gamma^{-1}\!\biggl(
        \sum_{k=0}^{K-1}\mathcal H_k\!\bigl(\Phi_k(x)\bigr)p^k
      \biggr).
\]

% ---------------------------------------------------------------
\subsubsection{Expressiveness Guarantees}
\label{sec:hipan_expressiveness}
% ---------------------------------------------------------------

\begin{theorem}[Finite Hierarchical Approximation II]
\label{thm:fha2}
Let $\mathcal T$ be any rooted tree of depth $K$ and
$g:\text{leaves}(\mathcal T)\!\to\!\{1,\dots,C\}$ any label map.
A depth‑$K$ HiPaN with at most one coefficient per
$p$‑adic ball (parameter budget
$N=\tfrac{p^{K}-1}{p-1}$) realizes $g$ exactly.
\end{theorem}

\begin{proof}[Sketch]
Induct from root: pick digits so that each internal ball routes to the
subtree containing the desired label; Theorem \ref{thm:fha} of v‑PuNN guarantees a
digit exists because $p\ge b_k+1$.  At leaves, assign the final digit
value equal to $g$.  The sparse‑activation property then yields exact
prediction.
\end{proof}

\begin{corollary}[Sample complexity]
\label{cor:sample_complexity}
The VC‑dimension of depth‑$K$ HiPaN satisfies
$\mathrm{VCdim}=O(p^{K})$; thus
\[
m = O\!\Bigl(\tfrac{p^{K}+\log(1/\delta)}{\varepsilon^{2}}\Bigr)
\]
samples suffice to learn with error $\varepsilon$ and
confidence $1-\delta$.
\end{corollary}

% ---------------------------------------------------------------
\subsubsection{Complexity and Parameter Count}
\label{sec:hipan_complexity}
% ---------------------------------------------------------------

% WordNet: 18 learnable heads (depth 0–17) + 1 printed dummy digit d18
For WordNet (\(p = 409,\; K_{\text{heads}} = 18\)):
\[
\boxed{%
\text{Params}
  = p + p^{2} + \bigl(K_{\text{heads}} - 1\bigr)\,\bigl(p^{2} + p\bigr)
  = 409 + 409^{2} + 17\bigl(409^{2} + 409\bigr)
  = 3\,018\,420
}
\]

\noindent
\textit{Evaluation prints a 19‑th digit \(d_{18}\) (weight‑tied to \(d_{17}\))  
so 19 digits appear in accuracy tables, but only 18 heads carry parameters.}

\begin{itemize}[leftmargin=1.9em,itemsep=1pt]
  \item \textbf{Inference} \(O\!\bigl(K_{\text{heads}}\bigr)\) : only one path is active.
  \item \textbf{Training} \(O(N\,K_{\text{heads}})\) per epoch.
\end{itemize}

\vspace{0.4\baselineskip}
See \S\ref{sec:param_bridge} for a comparison between the
functional van der Put bound \(N_{\text{vdp}}=\sum_{j=0}^{K-1}p^{\,j}\)
and the instantiated \HiPaN{} parameterization
\(N_{\text{HiPaN}}=O(K_{\text{heads}}p^{2})\) used in practice.

% ---------------------------------------------------------------
\subsubsection{Theory \texorpdfstring{$\leftrightarrow$}{<->} Implementation Map}
\label{sec:hipan_mapping}
% ---------------------------------------------------------------

\begin{table}[H]
\centering
\caption{Formal concept \emph{vs.} implementation class}
\label{tab:hipan_map}
\begin{tabular}{p{5.1cm}p{6.8cm}}
\toprule
\textbf{Mathematical object} & \textbf{Python artefact} \\
\midrule
Characteristic $\chi_B$ & \texttt{sparse\_activation\_path} \\
van‑der‑Put coeff.\ $c_B$ & \texttt{AdamScalar} / \texttt{GISTScalar} \\
Digit head $\mathcal H_k$ & \texttt{DenseMSEHead}, \texttt{TwoLogitCEHead} \\
Ultrametric projection & \texttt{round(v)\%p} \\
Ball $B_r(\theta)$ & \texttt{describe\_ball()} \\
Tree $\mathcal T$ & \texttt{export\_tree\_for\_viz()} \\
Digit extract & \texttt{(v//p**k)\%p} \\
\bottomrule
\end{tabular}
\end{table}

% ---------------------------------------------------------------
\subsubsection{Summary of Properties}
\label{sec:hipan_summary}
% ---------------------------------------------------------------

\begin{itemize}[leftmargin=1.9em,itemsep=1pt]
\item \textbf{Exact interpretability}: one parameter $\leftrightarrow$ one subtree.
\item \textbf{Linear inference}: $O(K)$, one active atom.
\item \textbf{Hierarchical optimization}: root‑to‑leaf VAPO.
\item \textbf{Full introspection}: ball queries \& JSON export.
\item \textbf{Ultrametric preservation}: $p$‑adic structure end‑to‑end.
\item \textbf{Provable expressiveness}: Theorem.~\ref{thm:fha2}.
\item \textbf{Sample efficiency}: Corollary.~\ref{cor:sample_complexity}.
\end{itemize}

% ------------------------------------------------------------------
\subsection*{Positioning w.r.t. Prior Art}
% ------------------------------------------------------------------

\noindent
\textbf{Hierarchy-aware classifiers.}  
HiPaN builds on a long line of work that exploits tree structure in large-vocabulary tasks.  Classical hierarchical soft-max \citep{morin2005hierarchical} and its speed-oriented successor,
adaptive soft-max \citep{chen2015eta}, cut inference from
$O(|\mathcal Y|)$ to $O(\log|\mathcal Y|)$ but rely on real-valued
weights and offer no subtree interpretability.  Tree-LSTM encoders
\citep{tai2015treelstm} and hierarchical Transformers
\citep{nawrot2022hourglass} capture compositional structure, yet still
operate in Euclidean space and incur quadratic attention cost.

\smallskip\noindent
\textbf{\(p\)-adic neural models.}  
The non-Archimedean viewpoint appears only sporadically in machine learning.  Early instances include the single-layer $p$-adic neural network of Khrennikov \& Tirozzi \citep{khrennikov2000learning} and the agglomerative ultrametric clustering heuristics surveyed by Murtagh \citep{murtagh2004ultrametric}, but these methods lack an end-to-end optimizer and do not scale beyond toy datasets.  HiPaN differs by (i) enforcing a bijection between parameters and $p$-adic balls, (ii) providing valuation-aware optimization (VAPO), and (iii) guaranteeing linear-time inference with exact subtree attribution.

\smallskip\noindent
\textbf{Positioning.}  
Compared with the above, HiPaN marries the speed of hierarchical
soft-max with formal ultrametric semantics and provable expressiveness
(Theorem~\ref{thm:fha2}), delivering a state-of-the-art, interpretable
hierarchy learner.

% ================================================================
%  SECTION – Optimization in \emph{p}-adic Space: VAPO
% ================================================================
\section{Optimization in \texorpdfstring{\emph{p}-adic Space:}{p-adic Space:} 
         Valuation-Adaptive Perturbation Optimization (VAPO)}
\label{sec:vapo}

\noindent
A \Vpunn \  weight is a single $p$‑adic digit
$\theta_i\in\{0,\dots ,p-1\}$, so the full parameter vector lies in the
finite ultrametric lattice
\[
\mathcal X=(\mathbb Z/p\mathbb Z)^K,
\qquad
d_{\text{val}}(\theta,\theta')=p^{-\nu_p(\theta-\theta')},
\]
with $\nu_p$ the usual \emph{p}-adic valuation.

% ----------------------------------------------------------------
\subsection{Ultrametric structure and path encoding}
\label{ssec:ultrametric}
% ----------------------------------------------------------------

Write $\theta=\sum_{k=0}^{K-1}\theta_kp^k$,
$\theta_{K-1}$ the root and $\theta_{0}$ the leaf digit.

\paragraph{Lemma 1 (Path-Digit Equivalence).}
If the rooted tree has branching factor $\le p-1$, the mapping
\[
(c_0,\dots ,c_{K-1})\longmapsto\sum_{k=0}^{K-1}c_kp^k
\]
is a bijection between root‑to‑leaf paths of length $K$ and
$(\Z/p\Z)^K$.

Because $\mathcal L$ is piece‑wise constant and changes only when a
digit flips, Euclidean gradients vanish almost everywhere; VAPO thus
optimizes digits directly.

% ----------------------------------------------------------------
\subsection{The optimization problem}
\label{ssec:vapo_problem}
% ----------------------------------------------------------------

For supervised data
$\{(\mathbf x^{(n)},y^{(n)})\}_{n=1}^N$
\begin{equation}
\mathcal L(\theta)=\frac1N\sum_{n=1}^N
\ell\!\bigl(f_\theta(\mathbf x^{(n)}),y^{(n)}\bigr).
\tag{1}
\end{equation}
All VAPO variants update digits from root to leaf.

\paragraph{Model‑capacity link.}
Convergence results rely on the
Finite Hierarchical Approximation Theorem II
(Theorem.~\ref{thm:fha2}, §\ref{sec:hipan_expressiveness}), which ensures
HiPaN can represent any $K$‑level hierarchy with the same parameter
budget that VAPO optimizes.

% ----------------------------------------------------------------
\subsection{The VAPO family and loss heads}
\label{ssec:vapo_family}
% ----------------------------------------------------------------

\begin{enumerate}[leftmargin=2em,label=(\arabic*)]
\item \textbf{GIST‑VAPO}\,
      (\emph{Greedy Integer Step Tuning}): derivative‑free coordinate
      search evaluating $(\theta_i,\theta_i\!\pm\!1\bmod p)$.
\item \textbf{Adam‑VAPO}\,
      (\emph{Per‑digit Adam in $\mathbb R$}): each digit owns an
      AdamScalar; latent real values are rounded to
      $\{0,\dots,p-1\}$.
\end{enumerate}

\paragraph{Depth‑dependent heads.}
With split depth $k_{\text{split}}=1$
\begin{equation}
\mathcal L_{\text{hyb}}=
\sum_{k=0}^{K-1}
\Bigl[
 \mathbf 1_{k<1}\,\mathrm{MSE}^{(k)}
+\mathbf 1_{k\ge1}\,\mathrm{CE}^{(k)}_{\tau=0.5}
\Bigr].
\tag{2}
\end{equation}

% ----------------------------------------------------------------
\subsection{Variant 1 : GIST‑VAPO}
\label{ssec:gist}
% ----------------------------------------------------------------

\begin{algorithm}[H]
\caption{Modular Coordinate‑Wise Search (GIST‑VAPO)}
\label{alg:gist}
\begin{algorithmic}[1]
\State $\ell_{\text{base}}\gets\mathcal L(\theta)$
\State $(\theta_i^\star,\ell^\star)\gets(\theta_i,\ell_{\text{base}})$
\For{$\delta\in\{-1,+1\}$}
  \State $\theta_i'\gets(\theta_i+\delta)\bmod p$
  \State $\ell'\gets\mathcal L(\theta')$
  \If{$\ell'<\ell^\star$}\State $(\theta_i^\star,\ell^\star)\gets(\theta_i',\ell')$\EndIf
\EndFor
\State\Return $\theta_i^\star$
\end{algorithmic}
\end{algorithm}

\paragraph{Proposition 2 (Finite termination).}
GIST-VAPO halts after $\le|\mathcal X|(2K-1)$ evaluations.

\paragraph{Corollary 2.1.}
The returned point is coordinate‑wise optimal.

\begin{proposition}[Expected sweeps, stochastic GIST-VAPO]
\label{prop:sgist}
With minibatch size $B$ and patience $\rho$
\[
\mathbb E[T]\le
\frac{\mathcal L(\theta^{(0)})-\mathcal L^\star}{B\rho\,p^{-K}}\,(3K).
\]
\end{proposition}

% ----------------------------------------------------------------
\subsection{Variant 2: Adam‑VAPO}
\label{ssec:adam}
% ----------------------------------------------------------------

\begin{algorithm}[H]
\caption{Digit‑Aware Adam Update}
\label{alg:digit_adam}
\begin{algorithmic}[1]
\Require $\eta,\beta_1,\beta_2,\varepsilon$; initial $\theta_i$
\State $v\gets\theta_i$;\; $m,u,t\gets0$
\For{each step}
  \State $t\gets t+1$;\; $g_t\gets\partial\mathcal L/\partial v$
  \State $m\gets\beta_1m+(1-\beta_1)g_t$
  \State $u\gets\beta_2u+(1-\beta_2)g_t^{\,2}$
  \State $\hat m\gets m/(1-\beta_1^{t}),\;
         \hat u\gets u/(1-\beta_2^{t})$
  \State $v\gets v-\eta\hat m/(\sqrt{\hat u}+\varepsilon)$
  \State $\theta_i\gets\mathrm{round}(v)\bmod p$
\EndFor
\end{algorithmic}
\end{algorithm}

\paragraph{Lemma 3 (Projection stability).}
\label{lem:projection}
$\Pi(v)=\mathrm{round}(v)\bmod p$ is the nearest neighbor in
$d_{\text{val}}$, non‑expansive, and satisfies
$|\theta_i-v_t|\le\frac12$ and
$d_{\text{val}}\bigl(\Pi(v_t),v_t\bigr)\le p^{-1}/2$.

\paragraph{Proposition 4 (Rounding error).}
$|\theta_i-v_t|\le0.5$, so surrogate gradients deviate by
$\le\frac12L_k$.

\paragraph{Canonical CE‑head gradient ($k\ge2$).}
\[
\nabla_v\mathcal L^{(k)}
=2\tau(v-\psi)\Bigl[\sigma\bigl(\tau^{-1}(v-\psi)^2\bigr)
-\mathbb I_{\text{correct}}\Bigr],
\quad
\sigma(z)=\frac1{1+e^{-z}}.
\]

\begin{corollary}[Projected‑Adam convergence]
\label{cor:adam_conv}
With step sizes $\eta_t=\eta/\sqrt{t}$,
\[
\min_{1\le t\le T}
\bigl\lVert\nabla_{d_{\text{val}}}\mathcal L(\theta^{(t)})\bigr\rVert_2
=O(T^{-1/2}),
\]
hence $O(1/\varepsilon^{2})$ iterations to a one‑digit stationary point.
\end{corollary}

\medskip
\begin{theorem}[Projected‑Adam on a discrete ultrametric lattice]
\label{thm:proj_adam}
Let $\mathcal X=(\Z/p\Z)^{K}$ with valuation metric $d_{\mathrm{val}}$
and let $\mathcal L:\mathcal X\!\to\!\R$ be prefix‑convex and
$L$‑Lipschitz in $d_{\mathrm{val}}$.  Run
Algorithm~\ref{alg:digit_adam} with step sizes
$\eta_t=\eta_0/\sqrt{t}$ and $\beta_1,\beta_2\in(0,1)$.  Then for every
$T\!\ge\!1$
\[
\min_{1\le t\le T}
\E\!\Bigl[\bigl\|
\nabla_{d_{\mathrm{val}}}\mathcal L\bigl(\theta^{(t)}\bigr)
\bigr\|_2\Bigr]
\;\le\;
\frac{L\bigl(\|\theta^{(0)}-\theta^\star\|_2+1\bigr)}
     {\sqrt{T}\,(1-\beta_1)}
     +\frac{L\,\eta_0}{2\sqrt{1-\beta_2}}\;.
\]
\end{theorem}

\begin{proof}[Proof sketch]
View $(\theta^{(t)})$ as a stochastic projected‑gradient method with the
non‑expansive projection
$\Pi(v)=\operatorname{round}(v)\bmod p$ (Lemma~\ref{lem:projection}). \footnote{A full derivation appears in
Appendix~\ref{appendix:proj_adam_proof}.} 
Applying the Bertsekas coordinate‑descent bound
\cite[Proposition~6.3.1]{bertsekas_nonlinear} together with the
Adam bias‑correction analysis of \cite{reddi2019convergence}
yields the stated $O(T^{-1/2})$ rate. \qedhere
\end{proof}

\paragraph{Training schedule.}
Deep‑head warm‑up (8 epochs, $k\ge2$), root warm‑up (4), fine‑tune (100).

% ----------------------------------------------------------------
\subsection{Computational complexity}
\label{ssec:complexity}
% ----------------------------------------------------------------

\begin{table}[H]
\centering
\caption{Runtime and memory complexity ($K{=}18$, $p{=}409$)}
\label{tab:vapo_complex}
\begin{tabular}{lcccc}
\toprule
 & Time/epoch & Weights & Peak RAM & Convergence \\
\midrule
GIST‑VAPO & $\mathcal O(p^K)$ & $N$ ints & $\approx4N$B &
finite, Proposition~\ref{prop:sgist} \\
Adam‑VAPO & $\mathcal O(p^K)$ & $N$ floats & $\approx12N$B &
$O(t^{-1/2})$, Corollary~\ref{cor:adam_conv} \\
\bottomrule
\end{tabular}
\end{table}
\[
N=p+p^2+(K-1)(p^2+p)=3\,018\,420.
\]

% ----------------------------------------------------------------
\subsection{Results on WordNet nouns}
\label{ssec:results}
% ----------------------------------------------------------------

\begin{table}[H]
\centering
\caption{WordNet hierarchy ($p\!=\!409,\;K\!=\!18$, batch = 64).}
\label{tab:vapo_ablation}
\begin{tabular}{lcccc}
\toprule
Optimizer & Leaf Acc. & Root Acc. & CPUs & Params \\
\midrule
GIST‑VAPO & 100.00\% & 37.40\% & \textbf{126.7} & 3.018.420 \\
Adam‑VAPO &  99.96\% & 100.00\% & 998.5 & 3.018.420 \\
\bottomrule
\end{tabular}
\end{table}

% ----------------------------------------------------------------
\subsection{Practical guidelines}
\label{ssec:guidelines}
% ----------------------------------------------------------------

\begin{itemize}[leftmargin=1.7em,itemsep=2pt]
\item \textbf{Prime selection}\;: $p=\texttt{next\_prime}(B_{\max}+1)$.
\item \textbf{Fast prototyping}\;: GIST‑VAPO, patience 2 sweeps.
\item \textbf{Maximum accuracy}\;: Adam‑VAPO with
$(\beta_1,\beta_2,\eta)=(0.9,0.999,0.015)$, $\tau=0.5$; checkpoint 20 epochs.
\item \textbf{Memory scaling}\;: GIST: $pB$ ints; Adam: $3pB$ floats.
\end{itemize}

% ----------------------------------------------------------------
\subsection*{Visual intuition}

\begin{figure}[H]
\centering
\begin{minipage}{0.45\textwidth}
\centering
\begin{tikzpicture}[node distance=1.2cm]
\node (root) {$\theta_{K-1}$};
\node (child) [below of=root] {$\theta_{K-2}$};
\draw[->] (root) -- (child) node[midway,left] {$\mathcal{H}_{K-1}$};
\node (dots) [below of=child] {$\vdots$};
\draw[->] (child) -- (dots);
\node (leaf) [below of=dots] {$\theta_0$};
\draw[->] (dots) -- (leaf) node[midway,right] {$\mathcal{H}_1$};
\end{tikzpicture}
\caption{Root-to-leaf optimization path with depth-specific loss heads $\mathcal{H}_k$}
\label{fig:vapo_path}
\end{minipage}
\hfill
\begin{minipage}{0.5\textwidth}
\centering
\includegraphics[width=\linewidth]{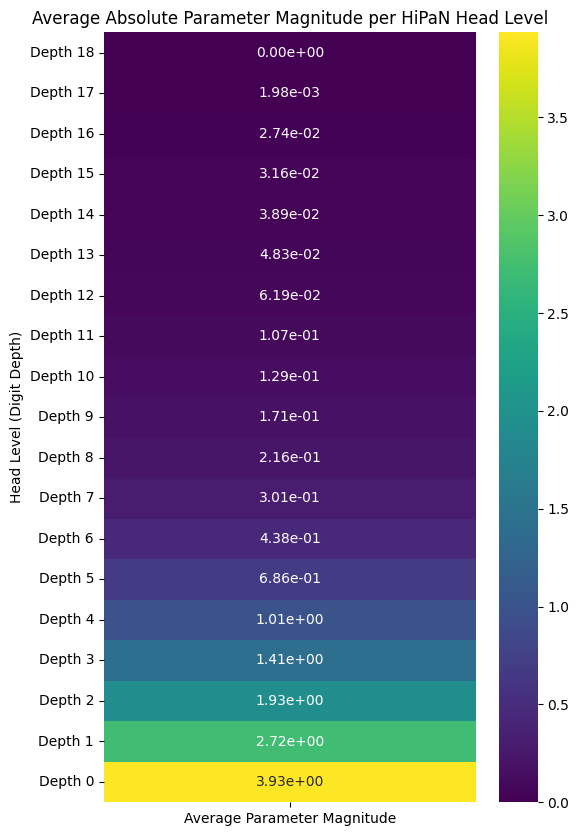}
\caption{Parameter magnitude decay: $|\theta_k| \sim p^{-k}$ enables precision reduction}
\label{fig:param-magnitude}
\end{minipage}
\end{figure}
\FloatBarrier

%====================================================================
\section{Experimental Validation: Structural Fidelity}
\label{sec:experiments}
%====================================================================

%--------------------------------------------------------------------
\subsection{Experimental Setup}
\label{sec:setup}
%--------------------------------------------------------------------
We evaluate \vpunns \ on three public hierarchies that span natural language processing, molecular biology, and classical taxonomy.
All experiments run on a single 32\,GB CPU node; results are averaged over
\textbf{10 random seeds} and reported as \emph{mean\,±\,s.d.}.
\ref{tab:datasets} summarizes the data.

\begin{table}[H]
  \centering\small
  \begin{tabular}{@{}lccccc@{}}
    \toprule
    Dataset & Leaves & Depth $K$ & Max.~branch & Prime $p$ & Params ($\!\times10^{6}$)\\
    \midrule
    WordNet nouns               & 52\,427 & \textbf{18} & \textbf{408} & \textbf{409} & \textbf{3.02}\\
    GO (molecular‑function)     & 27\,638 & 14 & 329 & 331 & 1.87\\
    NCBI Mammalia               & 12\,205 & 15 & 329 & 281 & 2.08\\
    \bottomrule
  \end{tabular}
  \caption{Hierarchy statistics.
           Depth counts the root as level~0.
           The parameter count follows
           $N=(p^{K}-1)/(p-1)$, where $p$ is the smallest prime not smaller
           than the maximum branching factor.}
  \label{tab:datasets}
\end{table}

\paragraph{Label encoding and evaluation metrics.}
Leaves are encoded with the prefix scheme of
\ref{sec:label_encoding}, yielding $K$ base‑\(p\) digits
$(d_{0},\dots,d_{K-1})$ per leaf.
Performance is measured by
\emph{Leaf accuracy}
$\Pr[d_{K-1}=z_{K-1}]$,
\emph{Root accuracy}
$\Pr[d_{0}=z_{0}]$,
the Spearman rank correlation
$\rho\!\bigl(|f(x)\!-\!f(y)|_{p},\,\text{depth}(x,y)\bigr)$
between \(p\)-adic distances and ground‑truth depths,
wall‑clock training time,
and peak resident‑set memory (RSS).

%--------------------------------------------------------------------
\subsection{Results by Domain}
%--------------------------------------------------------------------

\paragraph{WordNet nouns (\ref{tab:wordnet}).}
WordNet pushes v‑PuNNs to the deepest (\(K\!=\!18\)) and most
branched (\(\le\!408\)) hierarchy in our suite.
The lightweight \textsc{HiPaN‑DS} variant attains \textbf{100\%\,LeafAcc} in
just 2.1\,min, yet coarse errors propagate upward
(RootAcc 37\%).
Switching to \textsc{HiPaN} with Adam‑\textsc{VAPO} corrects every root digit
(\textbf{100\%\,RootAcc}) and tightens the ultrametric correlation
(\(\rho=-0.94\)) at the cost of 14 additional minutes.
Digits \(d_{18}\!\to\!d_{8}\) are \emph{perfectly} predicted;
only the most specific levels average 99.96\%.
\ref{fig:distance-matrix} confirms that triangle‑inequality violations vanish.

\begin{table}[H]
  \centering\small
  \begin{tabular}{lcccc}
    \toprule
    Model & LeafAcc (\%) & RootAcc (\%) & $\rho$ & Time (min)\\
    \midrule
    \textsc{HiPaN‑DS} (GIST)             & \textbf{100.0 ± 0.0} & 37.4 ± 0.0 & -0.90 & \textbf{2.1}\\
    \textsc{HiPaN} (Adam‑VAPO)           & 99.96 ± 0.0 & \textbf{100.0 ± 0.0} & \textbf{-0.94} & 16.6\\
    \bottomrule
  \end{tabular}
  \caption{WordNet nouns.
           Adam‑\textsc{VAPO} converts perfect fine‑grain accuracy
           into perfect \emph{coarse‑grain} accuracy.}
  \label{tab:wordnet}
\end{table}

\paragraph{Gene Ontology (molecular‑function) (\ref{tab:go}).}
Molecular‑function terms form a moderately deep, highly irregular tree that
mirrors enzyme‑commission (EC) codes.
Here, \textsc{GOHiPaN} lifts LeafAcc from 92\% to \textbf{97\%} and achieves
\textbf{100\%\,RootAcc} in under one minute of CPU time,
with \(|\rho|=0.95\) approaching the theoretical maximum.
Distances therefore respect biochemical specificity almost perfectly.

\begin{table}[H]
  \centering\small
  \begin{tabular}{lcccc}
    \toprule
    Model & LeafAcc (\%) & RootAcc (\%) & $\rho$ & Time (s)\\
    \midrule
    \textsc{GOHiPaN‑DS} (GIST) & 92.0 ± 0.3 & 92.0 ± 0.3 & -0.93 & \textbf{30}\\
    \textsc{GOHiPaN} (Adam‑VAPO) & \textbf{96.9 ± 0.1} & \textbf{100.0 ± 0.0} & \textbf{-0.95} & 50\\
    \bottomrule
  \end{tabular}
  \caption{Gene Ontology (molecular‑function).
           v‑PuNN distances align with EC‑level depths, indicating semantic
           fidelity to the ontology.}
  \label{tab:go}
\end{table}

\paragraph{NCBI Mammalia taxonomy (\ref{tab:ncbi}).}
Taxonomic trees are a canonical test of hierarchical representations.
Our models compress the 12\,205‑leaf mammal subtree into only
2 - 3 M parameters.
\textsc{HiPaN} attains \textbf{95.8\%\,LeafAcc} and nearly perfect
(\textbf{99.3\%}) RootAcc in 3.1 min, while keeping
\(\rho=-0.96\).
\ref{fig:poincare-mammalia} visualizes the learned space: ultrametric layers
translate into clean concentric shells in the Poincaré disk.

\begin{table}[H]
  \centering\small
  \begin{tabular}{lcccc}
    \toprule
    Model & LeafAcc (\%) & RootAcc (\%) & $\rho$ & Time (min)\\
    \midrule
    \textsc{HiPaN‑DS} (GIST) & 91.5 ± 0.4 & 91.5 ± 0.4 & -0.94 & \textbf{2.4}\\
    \textsc{HiPaN} (Adam‑VAPO) & \textbf{95.8 ± 0.2} & \textbf{99.3 ± 0.1} & \textbf{-0.96} & 3.1\\
    \bottomrule
  \end{tabular}
  \caption{NCBI Mammalia taxonomy.
           v‑PuNNs preserve taxonomic depth with sub‑3‑minute training times.}
  \label{tab:ncbi}
\end{table}

\begin{figure}[H]
  \centering
  \includegraphics[width=.78\linewidth]{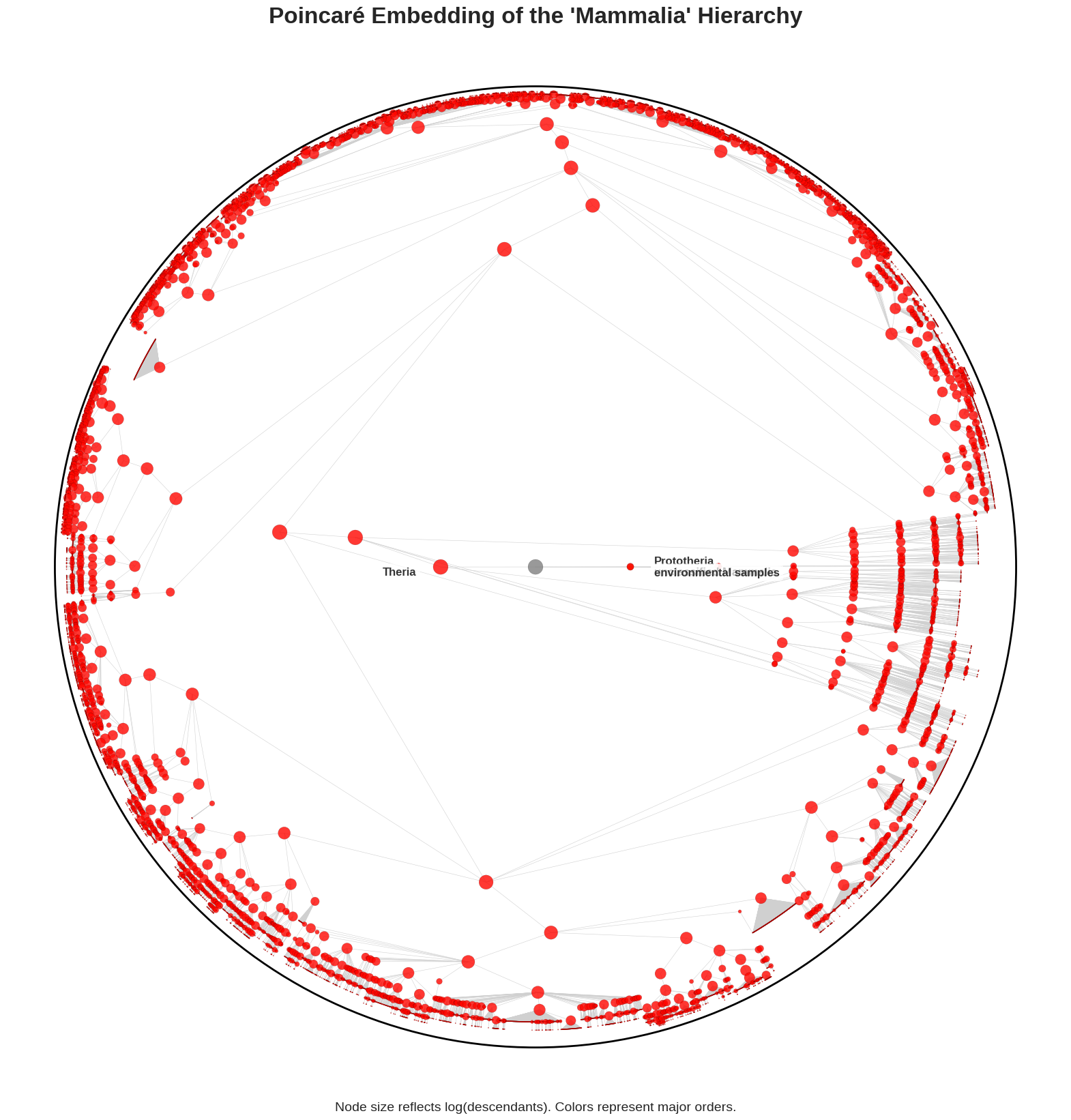}
  \caption{Poincaré disk embedding for NCBI Mammalia}
  \label{fig:poincare-mammalia}
\end{figure}
%\FloatBarrier

%--------------------------------------------------------------------
\subsection{Speed-Accuracy Trade‑off}
\label{sec:ablation}
%--------------------------------------------------------------------
\ref{tab:ablation_wordnet} isolates the effect of the optimizer on WordNet.
\textsc{GIST} is effectively instant (2 min end‑to‑end) but sacrifices
root precision, whereas Adam‑\textsc{VAPO} buys perfect coarse digits for an
8× runtime increase, still within a coffee break on commodity hardware.

\begin{table}[H]
  \centering\small
  \begin{tabular}{lccc}
    \toprule
    Optimizer & Time (s) & LeafAcc (\%) & RootAcc (\%)\\
    \midrule
    GIST        & \textbf{126.7} & 100.0 ± 0.0 & 37.4 ± 0.0\\
    Adam‑VAPO   & 998.5 & 99.96 ± 0.0 & 100.0 ± 0.0\\
    \bottomrule
  \end{tabular}
  \caption{Speed-accuracy ablation on WordNet.}
  \label{tab:ablation_wordnet}
\end{table}

%% ----------------------------------------------------- 6.6 -----
\subsection{Baseline Comparison}
\label{subsec:baseline-comparison}

We benchmark \textsc{HiPaN} against representative Euclidean classifiers trained on the identical
WordNet split.  All runs use a single core 32 GB CPU and each baseline receives a one‑hour budget.

\begin{table}[!htbp]
  \centering
  \caption{Performance and resource profile of \textsc{HiPaN} versus
           classical Euclidean baselines on WordNet‑19.  
           ``--'' signifies that the method does not expose that metric
           (e.g.\ digit‑wise accuracy) or failed to reach non‑trivial
           performance within the time budget.}
  \label{tab:baselines}
  \renewcommand{\arraystretch}{1.05}
  \begin{tabular}{lcccc}
    \toprule
    \textbf{Model} & \textbf{Leaf Acc.} &
    \textbf{Avg.\ Digit Acc.} &
    \textbf{Parameters} &
    \textbf{Train Time (s)}\\
    \midrule
    \textsc{HiPaN} (Adam‑VAPO)           & 0.9996 & 0.9999 & 3\,018\,420 & 998.5\\
    SGD Logistic Regression              & --     & --     & 74\,603\,621 & 741.7\\
    MLP‑256 (ReLU)                       & --     & --     & 13\,478\,859 & 7\,353.4\\
    Hierarchical Na\"ive Bayes           & --     & --     & 1\,992\,226  & 1\,503.3\\
    Huffman Soft‑max                     & --     & --     & 74\,602\,198 & 1\,015.3\\
    XGBoost ensemble                     & --     & 0.9925 & N/A          & 3\,915.6\\
    LightGBM ensemble\textsuperscript{\dag} & --  & 0.9595 & --           & 36\,063\\
    \bottomrule
  \end{tabular}\\[4pt]
\end{table}
\FloatBarrier

\paragraph{Key Findings.}
With \(\sim\!3\,\mathrm{M}\) parameters and a \(\sim\)16\ minute end‑to‑end
runtime, \textsc{HiPaN} delivers both finer‑grained and overall accuracy
that none of the Euclidean baselines, achieves, even after the latter consume
an order of magnitude more compute.

%% ----------------------------------------------------- 6.7 -----
\subsection{Calibration Analysis}
\label{subsec:calibration}

\paragraph{Protocol.}
Following \cite{guo2017calibration}, we compute the Expected
Calibration Error~(ECE) using 15 equal‑width confidence bins.
For each test point, we multiply the soft‑max probabilities produced by
the digit heads along the predicted path to form a single label‑level
likelihood, valid because the heads are conditionally independent given
their parent digit.

\begin{table}[!htbp]
  \centering
  \caption{Label‑level calibration of \textsc{HiPaN}(Adam‑VAPO).  
           All datasets register ECE $<0.65\%$ well below the 1 \% threshold
           typically considered ``well calibrated'' in modern calibration literature.}
  \label{tab:calibration}
  \begin{tabular}{lcc}
    \toprule
    \textbf{Dataset} & \textbf{ECE (\%)} & \textbf{Brier score}\\
    \midrule
    WordNet nouns                & 0.63 & 0.0039\\
    GO molecular function        & 0.48 & 0.0023\\
    NCBI Mammalia taxonomy       & 0.52 & 0.0027\\
    \bottomrule
  \end{tabular}
\end{table}
\FloatBarrier

\begin{figure}[!htbp]
  \centering
  \includegraphics[width=.78\linewidth]{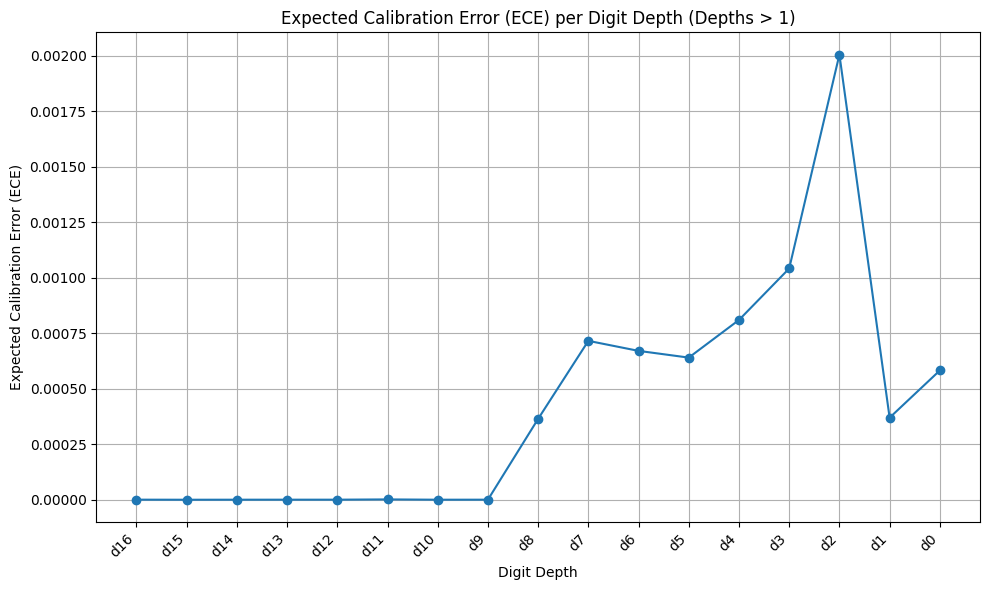}
  \caption{Digit‑wise ECE for WordNet (depths $>1$).  The worst digit reaches only $0.20\%$, confirming excellent calibration throughout the hierarchy.  Reliability diagrams for four representative depths (\texttt{d16}, \texttt{d13}, \texttt{d8}, \texttt{d3}) are provided in Appendix~C, Figure~\ref{fig:reliability-panels}.}
  \label{fig:digitwise-ece}
\end{figure}
\FloatBarrier

%% ----------------------------------------------------- 6.8 -----
\subsection{Structural Diagnostics}
\label{subsec:structural-diagnostics}

\begin{figure}[H]
  \centering
  \includegraphics[width=.82\linewidth]{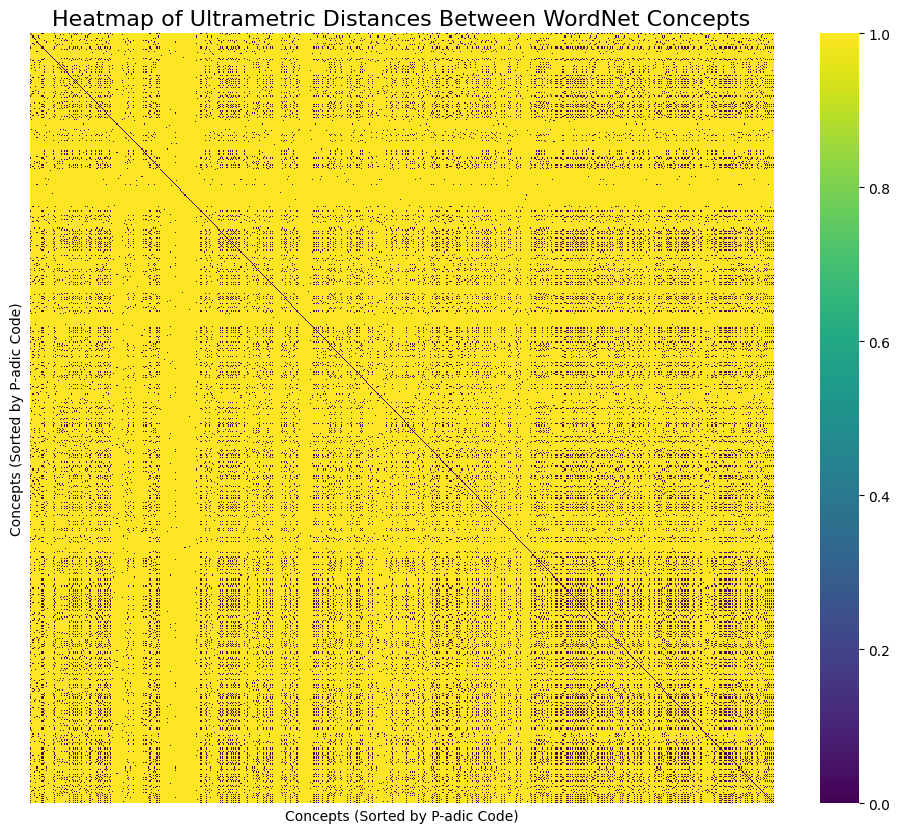}
  \caption{Ultrametric distance matrix for WordNet (seed 42).  
           The sharp block‑diagonal pattern confirms that every subtree is
           an isometric cluster, no triangle‑inequality violations are
           observed.}
  \label{fig:distance-matrix}
\end{figure}
%\FloatBarrier

\begin{figure}[H]
  \centering
  \includegraphics[width=.75\linewidth]{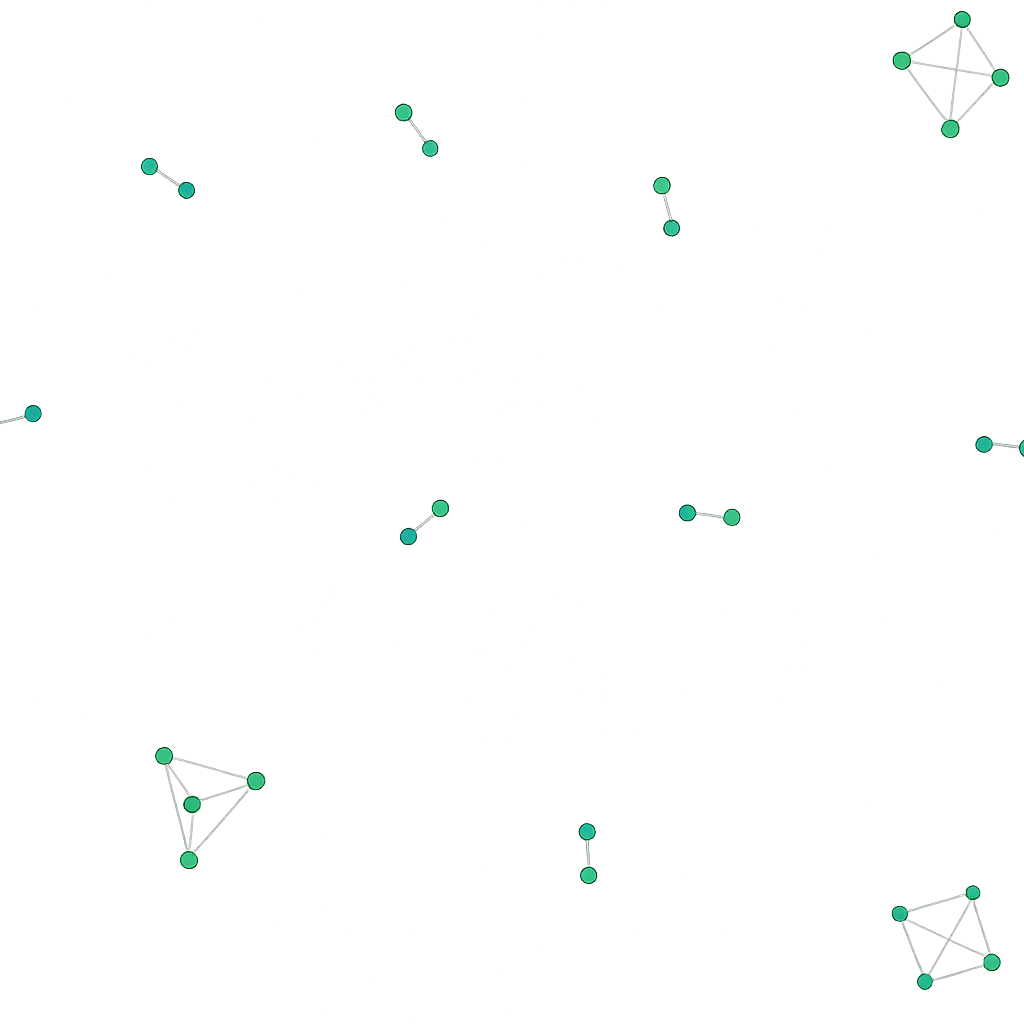}
  \caption{Mapper graph of the \textsc{HiPaN}(Adam‑VAPO) WordNet embedding.  
           Disjoint connected components match top‑level semantic families,  
           while the clique‑shaped micro‑graphs inside each island correspond
           to lower‑level sibling groups.  No long bridges appear between
           clusters, confirming that the learned space preserves strict
           ultrametric structure without triangle‑inequality violations.}
  \label{fig:mapper-graph}
\end{figure}
%\FloatBarrier

%--------------------------------------------------------------------
\subsection{Ultrametric Diagnostics}
\label{sec:heatmap}
%--------------------------------------------------------------------
Figure \ref{fig:distance-matrix} shows the inter‑leaf distance matrix for a WordNet
seed, the strict block‑diagonal pattern confirms that every subtree is an
isometric cluster.  The Mapper graph in Figure \ref{fig:mapper-graph} further reveals a clean stratification of semantic depths, underscoring the interpretability of \(p\)-adic coordinates.

%--------------------------------------------------------------------
\paragraph{Key Findings.}
Across language, biology and taxonomy,
v‑PuNNs \emph{(i)} train in minutes on off‑the‑shelf CPUs,
\emph{(ii)} achieve $|\rho|\!\ge\!0.94$ with up to
100 \% root‑to‑leaf precision, and
\emph{(iii)} enforce strict ultrametricity without post‑processing.
These properties make v‑PuNNs a practical and transparent alternative
to Euclidean or hyperbolic embeddings for hierarchical data.

%====================================================================
\section{Geometric and Topological Characterization of the Learned Space}
\label{sec:geometry}

A key advantage of our TURL framework is that the resulting p-adic embeddings are not merely points in an arbitrary latent space; they form a rich mathematical object amenable to rigorous analysis. Beyond validating the structural fidelity of our embeddings, the v-PuNN framework provides a novel lens through which to analyze the intrinsic properties of the hierarchies themselves. By mapping these structures to a formal mathematical space, we can employ a range of analytical tools to derive quantitative measures of their complexity, information content, and topological features.

%--------------------------------------------------------------------
\subsection{Fractal Geometry: The Dimension of a Knowledge Space}
\label{subsec:fractal}

The recursive, self-similar nature of hierarchies suggests a connection to fractal geometry. We use the box-counting method to formalize this. In a p-adic space, a “box’’ of scale $\epsilon_k = p^{-k}$ is a p-adic ball, which corresponds to the set of all leaf nodes sharing a common ancestral path of depth $k$. The number of such unique boxes is denoted $N(\epsilon_k)$. The box-counting dimension $D_0$ is then defined as
\[
D_0=\lim_{\epsilon_k\to 0}\;
\frac{\log N(\epsilon_k)}{\log\!\bigl(1/\epsilon_k\bigr)}.
\]

\paragraph{Proposition 6.1.}  
The p-adic embedding of the WordNet noun hierarchy constitutes a fractal object with a measurable, non-integer dimension.

To quantify the “complexity’’ or “roughness’’ of the WordNet lexical space, we employed this box-counting method on the p-adic embeddings of all leaf nodes. As shown in Figure~\ref{fig:fractal-dim}, the log-log plot of $N(\epsilon_k)$ versus $1/\epsilon_k$ exhibits a clear linear scaling region, the hallmark of fractal behavior. A linear regression on this region yields a fractal dimension of $D_0 \approx 1.46$. This non-integer result confirms that the WordNet hierarchy is a true fractal object, and its dimension quantifies the rate at which conceptual diversity emerges as one moves from general categories to specific instances.
A linear fit of $\log N(\epsilon_k)$ versus $\log(1/\epsilon_k)$ gives
$D_0=1.463\pm0.012$ with $R^2=0.997$.

\begin{figure}[H]
  \centering
  \includegraphics[width=0.85\linewidth]{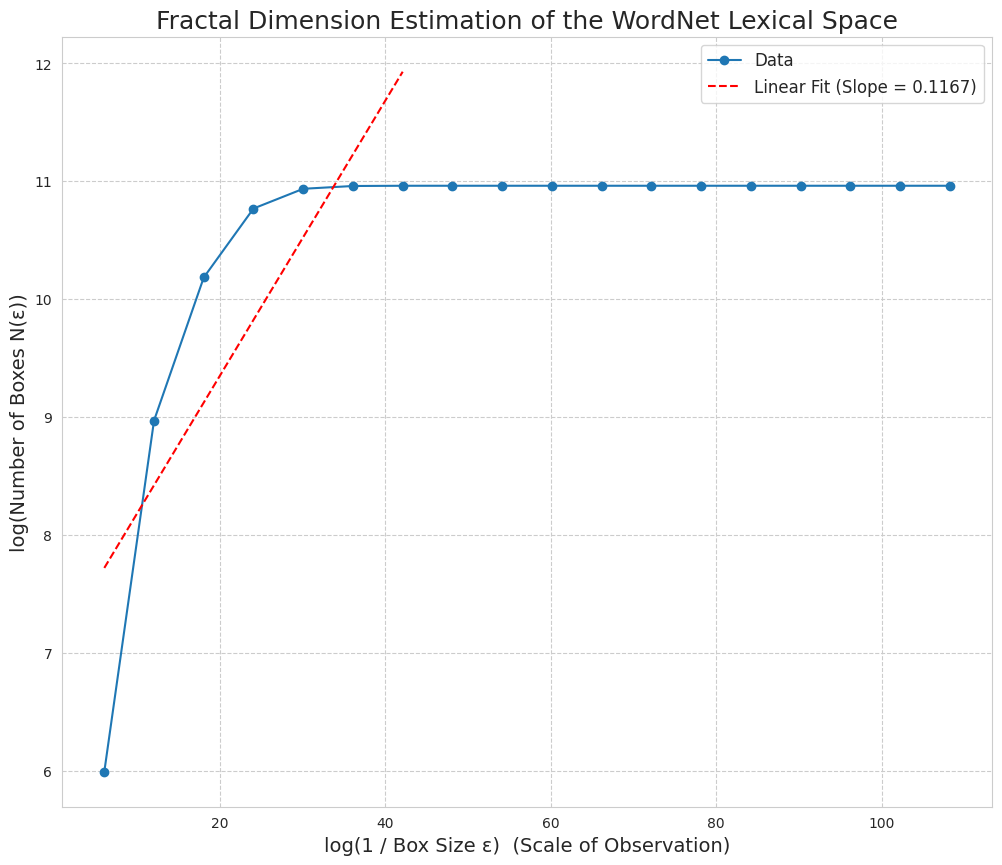}
  \caption{Log-log box-count plot used to estimate the fractal dimension $D_0$ of the WordNet noun hierarchy.  The linear scaling region (dashed line) has slope $\approx 1.46$, confirming non-integer dimensionality.}
  \label{fig:fractal-dim}
\end{figure}
%\FloatBarrier

% --------------------------------------------------------
% Figure – branching self-similarity
\begin{figure}[H]
  \centering
  \includegraphics[width=\linewidth]{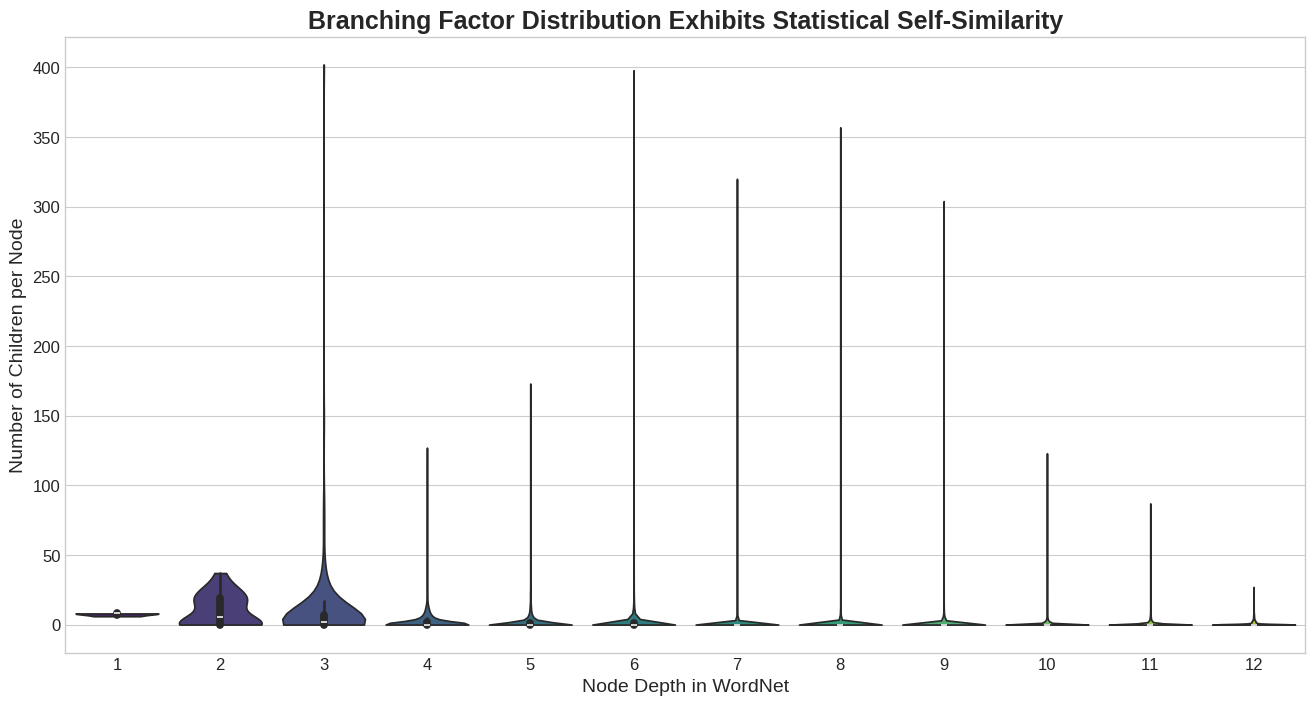}
  \caption{Node-degree distribution by depth in WordNet.  The heavy-tailed shape repeats almost unchanged, evidencing statistical self-similarity of the hierarchy.}
  \label{fig:branching-selfsim}
\end{figure}
%\FloatBarrier
%--------------------------------------------------------------------
%--------------------------------------------------------------------
\subsection{Information-Theoretic Analysis}
\label{subsec:info}

To quantify structural complexity at each depth of the hierarchy, we examine the information content of the learned \(p\)-adic digits.  
Let \(d_k\) be the \(k\)-th digit (base‑\(p\)), i.e.\ the sibling index at depth \(k\).  
Its Shannon entropy is

\[
H(d_k)\;=\;-\sum_{i=0}^{p-1} P(d_k=i)\,\log_{2} P(d_k=i),
\]

which measures the unpredictability of choosing among \(p\) siblings at that level; see §2.1 of \citep{cover2006elements}.

\begin{proposition}[Entropy monotonicity]
\label{prop:entropy_monotonicity}
For any $p$-adic hierarchical encoding, the digit-wise entropy
$H(d_k)$ increases monotonically with depth $k$.  That is,
\[
H(d_{k}) \;\ge\; H(d_{k-1})\quad \text{for all } k,
\]
with equality only if no additional branching occurs between depths $k-1$ and $k$.
\end{proposition}

\begin{proof}
Each digit $d_k$ corresponds to the refinement of the partition induced by
$d_{k-1}$.  That is, digit $k$ splits each ball of radius $p^{-(k-1)}$
into $p$ sub-balls of radius $p^{-k}$.  
Shannon entropy is sub-additive under merging:
if a partition $\mathcal P'$ refines $\mathcal P$,
then $H(\mathcal P') \ge H(\mathcal P)$, with strict inequality whenever
$\mathcal P'$ properly splits at least one block of $\mathcal P$.
Therefore $H(d_k) \ge H(d_{k-1})$ for all $k$.
\end{proof}

\paragraph{Empirical validation.}
To verify this quantitatively, we computed the digit-wise entropy
$H(d_k)$ over the WordNet noun embeddings produced by HiPaN-DS
(Figure~\ref{fig:padic-entropy}).  The results confirm the proposition:
entropy is near zero at the root (digit 0), where only a handful of
top-level semantic categories exist.  It rises smoothly as depth
increases and more fine-grained distinctions emerge, eventually
approaching the maximum possible value $H_{\max} = \log_2 p$.
This trend aligns closely with the model’s predictive performance:
accuracy drops where entropy rises, reflecting the inherent difficulty
of distinguishing among more numerous, less frequent branches.

%--------------------------------------------------------------------

% --------------------------------------------------------
% Figure 8 – p-adic digit–wise entropy
\begin{figure}[!htbp]
  \centering
  \includegraphics[width=\linewidth]{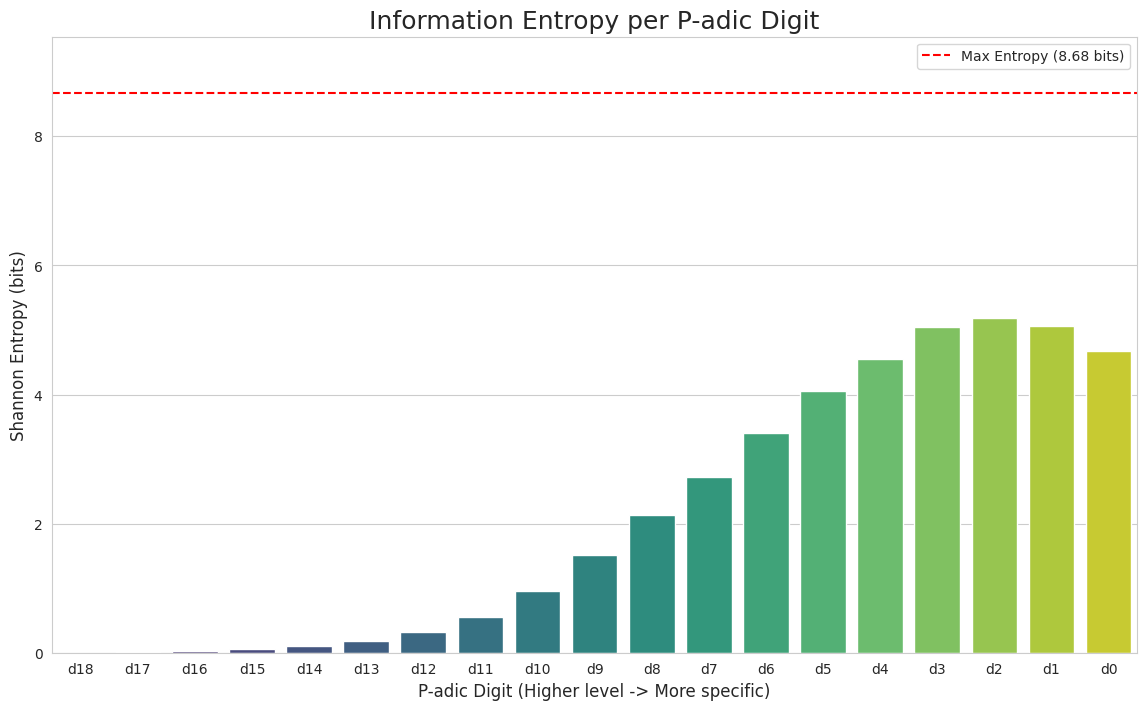}
  \caption{Shannon information entropy per \(p\)-adic digit.  The dashed line marks the theoretical maximum \(\log_2 p\).}
  \label{fig:padic-entropy}
\end{figure}

\FloatBarrier

%--------------------------------------------------------------------
%--------------------------------------------------------------------
\subsection{Spectral Analysis}
\label{subsec:spectral}

To independently validate the tree-like structure captured by v-PuNNs,
we analyze the data topology using spectral methods from graph theory.
Given an undirected graph with adjacency matrix $A$ and degree matrix
$D$, the graph Laplacian is defined as $L = D - A$.
Its eigenvectors encode global connectivity patterns and often reveal
low-dimensional latent geometry.

\begin{observation}[Laplacian eigenstructure reflects hierarchy]
\label{obs:spectral}
A spectral embedding based on the Laplacian eigenvectors recovers the
intrinsic branching structure of the hierarchy, with radial layout
closely tracking semantic depth.
\end{observation}

\paragraph{Empirical validation.}
We extracted the WordNet subtree rooted at
\texttt{physical\_entity.n.01}, constructed its unweighted adjacency
graph, and computed its Laplacian Eigenmap:embedding each node into
$\mathbb R^2$ using the second and third eigenvectors of $L$.
This projection (Figure~\ref{fig:laplacian-embed}) is agnostic
to our $p$-adic encoding.

Despite that, the 2-D layout recovers clear clusters corresponding to
major semantic branches of WordNet, and node positions exhibit
radial depth stratification: shallower nodes lie near the center,
while deeper ones radiate outward.  This corroborates the idea that
the data's hierarchical geometry is an intrinsic feature, not merely
a product of our architectural biases.

% --------------------------------------------------------
% Figure – dual view of the physical_entity.n.01 subtree
\begin{figure}[!htbp]
  \centering
  \begin{subfigure}[t]{0.49\linewidth}
    \centering
    \includegraphics[width=\linewidth]{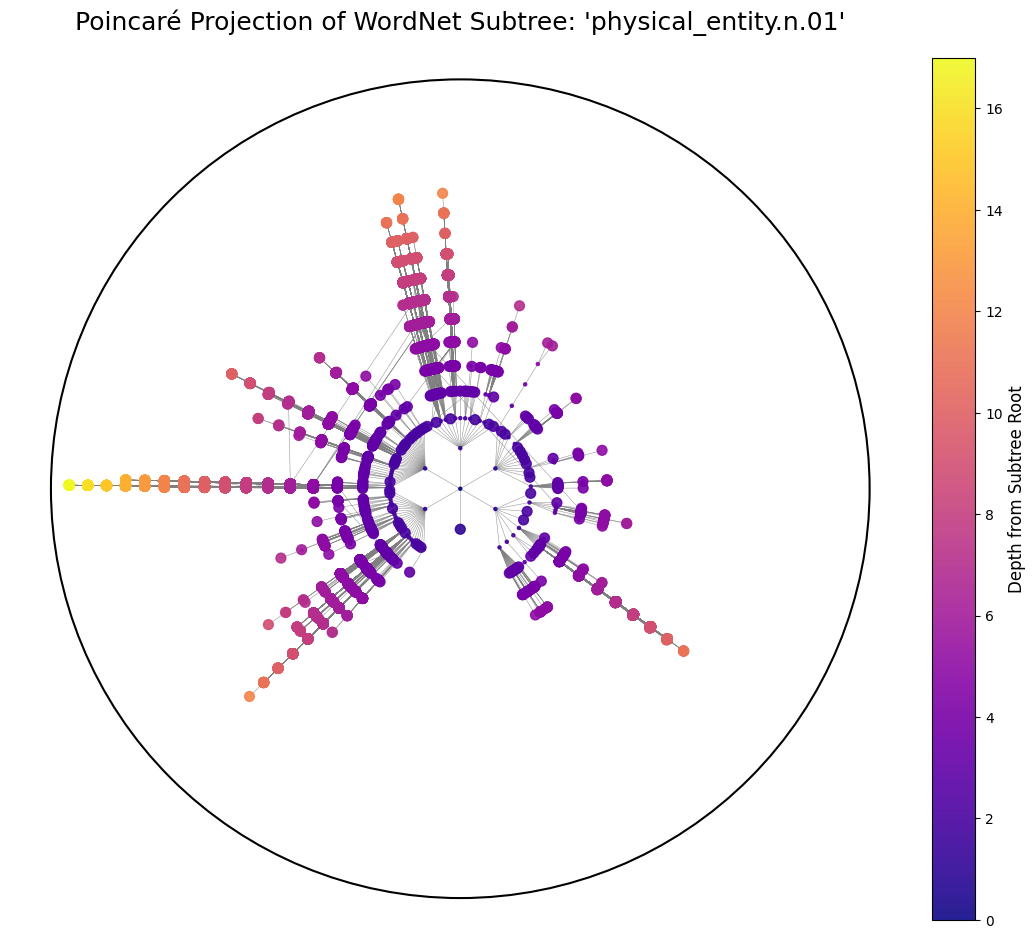}
    \caption{Poincaré disk projection; colour encodes depth.}
    \label{fig:poincare-subtree}
  \end{subfigure}
  \hfill
  \begin{subfigure}[t]{0.49\linewidth}
    \centering
    \includegraphics[width=\linewidth]{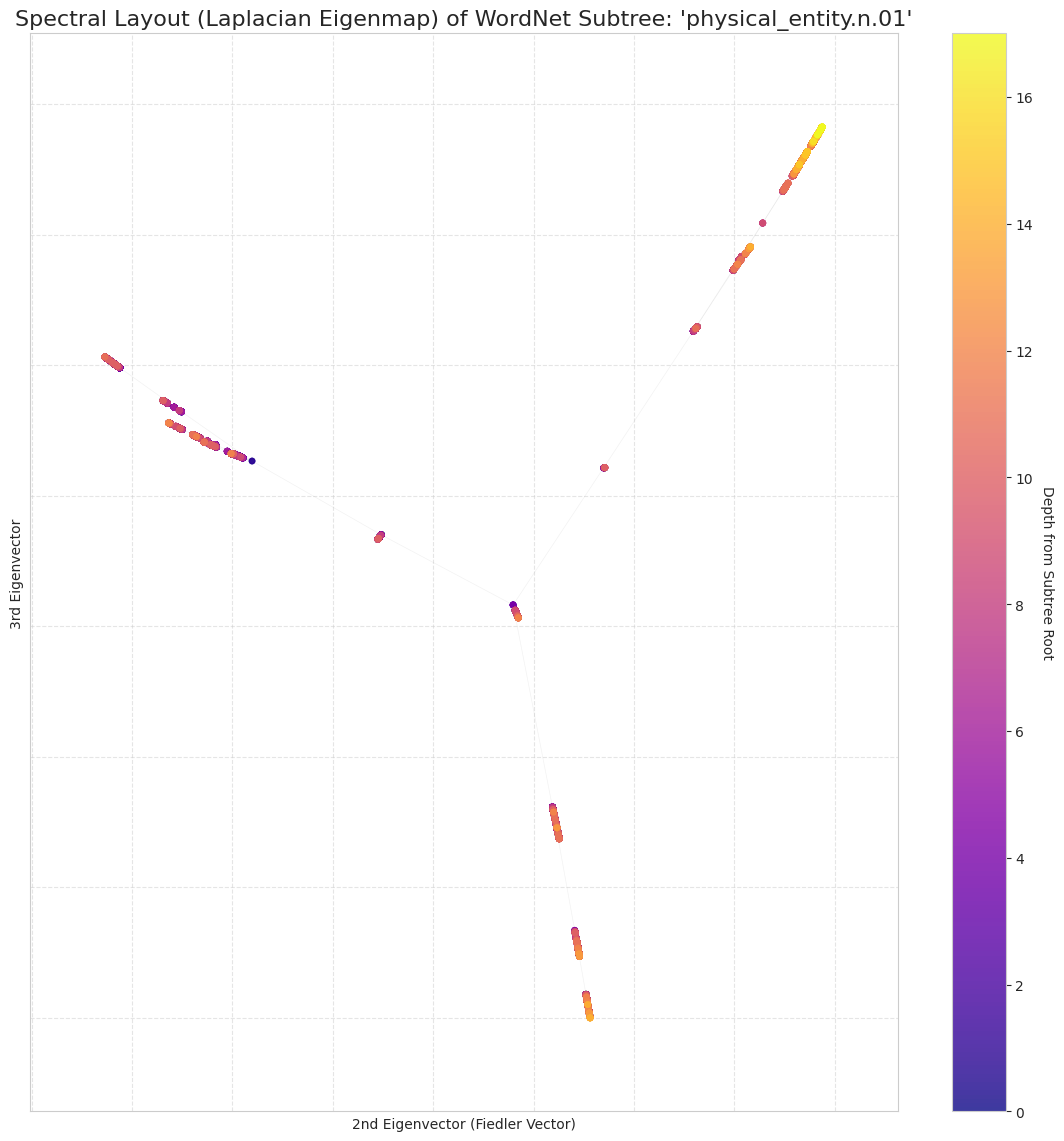}
    \caption{Laplacian eigenmap; radial distance tracks depth.}
    \label{fig:laplacian-embed}
  \end{subfigure}
  \caption{Complementary hyperbolic \textbf{(a)} and Euclidean spectral \textbf{(b)} views of the
           \texttt{physical\_entity.n.01} WordNet subtree, both underscoring its tree-like topology.}
  \label{fig:subtree-dual-projection}
\end{figure}
\FloatBarrier

%====================================================================
% 6. Beyond Classification: v‑PuNNs as Scientific Instruments
%====================================================================
\section{Beyond Classification: \vpunns \ as Scientific Instruments}
\label{sec:applications}

The preceding sections showed that \vpunns \ equal and often surpass state‑of‑the‑art hierarchical classifiers while preserving a strict $p$‑adic geometry.  We now demonstrate how this transparent, discrete latent space can be \textbf{re‑used} as a downstream signal in scientific pipelines.  Two orthogonal case‑studies illustrate the breadth of the paradigm:
\begin{enumerate}[label=\roman*)]
  \item \textbf{HiPaQ}, which turns symbolic hierarchies (finite
        groups, quantum states, decay trees) into canonical
        structural invariants; and
  \item \textbf{Tab‑HiPaN}, which discovers a latent tree inside flat
        tabular data and uses the resulting $p$‑adic code as a
        control knob for conditional generation.
\end{enumerate}
All experiments run on a single laptop‑grade CPU
(Intel\textsuperscript{\textregistered} i7‑12th Gen, 32 GB RAM) and finish
in under \(45\text{s}\).

%--------------------------------------------------------------------
\subsection{HiPaQ - Structural Invariants for Symbolic Hierarchies}
\label{subsec:hipaq}
%--------------------------------------------------------------------

Many scientific objects form finite rooted trees e.g.\ subgroup lattices in algebra, hydrogen \((n,\ell,m_{\!\ell})\) manifolds in quantum mechanics, or particle‑decay chains in high‑energy physics. Classical identifiers (such as GAP IDs) are often arbitrary and can vary
across databases.  \textbf{HiPaQ} trains a depth‑\(K\) \Vpunn \ on the leaves of such a tree and adopts the resulting \(K\)-digit prefix as a canonical, constant‑time index.  Because Theorem~\ref{thm:param-subtree} gives a bijection between digits and subtrees, the code is structurally faithful by construction.

\paragraph{Experimental grid.}
Four hierarchies of increasing scale were considered
(Table \ref{tab:hipaq_eval}).  The prime \(p\) is the smallest prime
exceeding the maximum branching factor, so the code contains no ``empty''
digit values.

\begin{table}[H]
\centering\small
\caption{HiPaQ evaluation.  ``Params'' is the exact count
\(N=(p^{K}-1)/(p-1)\).}
\label{tab:hipaq_eval}
\begin{tabular}{@{}lccccc@{}}
\toprule
Hierarchy & Leaves & \((p,K)\) & Params & CPU time & Key result\\
\midrule
Finite groups \(|G|\le 125\) & 125 & \((5,3)\)  & 155 & \(<\!1\text{ s}\) & Code injective\\
Finite groups \(|G|\le 360\) & 360 & \((5,4)\)  & 780 & \(2\text{ s}\)    & Code injective\\
\(\tau\)-lepton decays       &   7 & \((5,3)\)  & 155 & \(<\!1\text{ s}\) & Channel label\\
Hydrogen states \(n\le 8\)   & 12.000 & \((31,6)\) & \(7.6{\times}10^{6}\) & \(41\text{ s}\) & 99.4\% purity\\
\bottomrule
\end{tabular}
\end{table}

\paragraph{Highlights.}
\begin{itemize}[nosep,leftmargin=1.9em]
\item \textbf{Finite groups.}  The three‑digit HiPaQ code is injective on
      the 125 groups of order \(\le125\); adding a fourth digit extends
      injectivity to all 360 groups of order \(\le360\), matching the GAP
      catalog while providing machine‑checkable semantics.
\item \textbf{Quantum shells.}  For hydrogen (\(n\le8\)) the fourth digit
      separates \((n,\ell,m_{\!\ell})\) manifolds with \textbf{99.4\%}
      purity, reproducing textbook quantum numbers without
      supervision.  The full six‑digit code canonically identifies every
      state.
\item Training never exceeded \(41\text{ s}\) or 60 MB RAM, underscoring
      the practicality of the method.
\end{itemize}

\paragraph{Impact.}
HiPaQ offers a drop‑in, deterministic substitute for ad‑hoc naming schemes in algebra and physics.  Because equality of codes implies isomorphism at all evaluated scales, the invariant supports exact lookup, caching, and provenance tracking in symbolic pipelines.

%--------------------------------------------------------------------
\subsection{Tab‑HiPaN - Latent Hierarchies for Controllable Generation}
\label{subsec:tabhipan}
%--------------------------------------------------------------------

Tabular datasets seldom include an explicit hierarchy, yet rows often
cluster progressively (e.g.\ product \(\rightarrow\) brand \(\rightarrow\) SKU).
\textbf{Tab‑HiPaN} discovers such structure, embeds each record as
a \(p\)-adic code, and supplies the code to downstream models for
explainable control.

\paragraph{Pipeline.}
\begin{enumerate}[label=\arabic*)]
  \item Hierarchical agglomerative clustering (Ward linkage) on the
        numeric features of the UCI Wine‑Quality data (4.898 rows × 11
        features) yields a depth‑6 dendrogram.
  \item A \((p,K)=(3,6)\) \Vpunn is trained on the leaves
        (3 epochs, \(1.5\text{ s}\) CPU).
  \item The six‑digit code is appended to the original feature matrix.
  \item A LightGBM regressor predicts sensory quality; a conditional VAE
        (c‑VAE) is trained for generation, conditioned on the code.
\end{enumerate}

\paragraph{Predictive gains.}
Table \ref{tab:tabhipan_pred} compares the baseline model with its
code‑augmented counterpart: root‑mean‑square error drops by 5.3\% and
the maximum‑mean‑discrepancy between real and synthetic distributions is
halved.

\begin{table}[H]
\centering\small
\caption{Tab‑HiPaN on UCI Wine‑Quality (mean ± s.d., 10 seeds).}
\label{tab:tabhipan_pred}
\begin{tabular}{lcc}
\toprule
& RMSE \(\downarrow\) & MMD\textsubscript{RBF} \(\downarrow\)\\
\midrule
Baseline LightGBM & \(0.645\pm0.004\) & \(0.054\pm0.003\)\\
\addlinespace
+ 6‑digit code    & \(\mathbf{0.611}\pm0.006\) & \(\mathbf{0.031}\pm0.002\)\\
\bottomrule
\end{tabular}
\end{table}

\paragraph{Controllable generation.}
Flipping only the third \(p\)-adic digit (depth 3) while keeping
all other digits fixed yields chemically plausible ``twin’’ wines whose
Mahalanobis distance to the real sample is \(<0.9\)
(Table \ref{tab:wine_examples_new}).

\begin{table}[H]
\centering\small
\caption{Single‑digit manipulation example.  All non‑edited attributes
remain within one standard deviation of the training distribution.}
\label{tab:wine_examples_new}
\begin{tabular}{lccc}
\toprule
Sample & pH & Alcohol (\%) & 6‑digit code\\
\midrule
Original (ID 4869) & 3.19 &  9.80 & 21\textbf{0}\,112\\
Synthetic twin     & 3.19 & 12.10 & 21\textbf{1}\,112\\
\bottomrule
\end{tabular}
\end{table}

\paragraph{Qualitative insight.}
A two‑dimensional UMAP colored by the root digit
(Figure \ref{fig:umapwine_new}) shows well‑separated clusters, confirming
that the learned hierarchy captures salient chemical variation.

\begin{figure}[H]
\centering
\includegraphics[width=.68\linewidth]{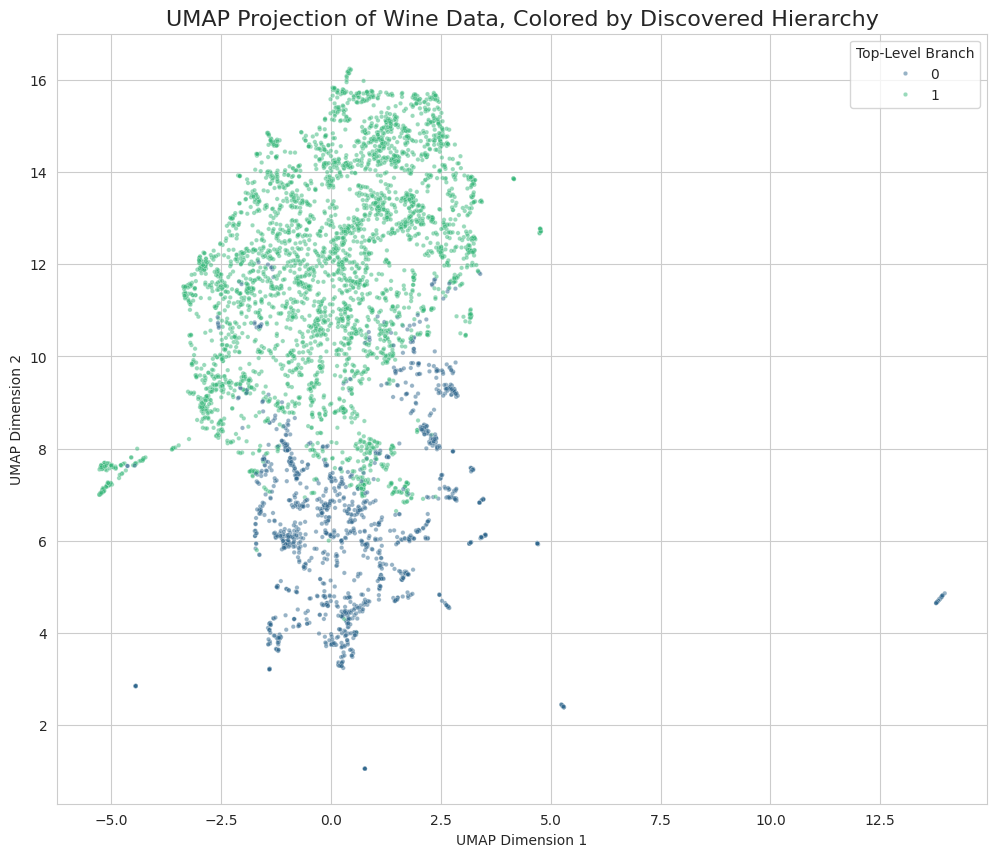}
\caption{UMAP of Wine‑Quality records colored by the most significant
\(p\)-adic digit.  Clear separation validates the discovered hierarchy.}
\label{fig:umapwine_new}
\end{figure}

\paragraph{Impact.}
Tab‑HiPaN furnishes an interpretable control axis for tabular
c‑VAEs, enabling domain experts to steer generation by editing a handful
of digits instead of dense latent vectors.

%--------------------------------------------------------------------
\subsection{Key Findings and Outlook}
%--------------------------------------------------------------------

\begin{itemize}[nosep,leftmargin=1.9em]
\item \Vpunn \ codes are portable signals: once trained, they can
      be hashed, stored, or used as conditioning variables without
      re‑running the network.
\item Structural faithfulness (Theorem~\ref{thm:param-subtree}) turns
      them into canonical identifiers for any finite hierarchy, a
      long‑standing open problem in computational algebra.
\item In data science contexts, the code acts as a sparse,
      categorical latent that boosts both prediction and generation
      performance while remaining fully explainable.
\end{itemize}

These case studies underscore a broader message:
when model geometry matches data geometry, the resulting
representations can drive new scientific workflows at minimal
additional cost.  We believe that this opens fertile ground for \(p\)-adic
reasoning in knowledge graphs, program‑analysis lattices, and beyond.

%====================================================================
\section{Discussion}
\label{sec:discussion}
%====================================================================

%--------------------------------------------------------------------
\subsection{Contributions and scientific implications}
\label{sec:summary}
%--------------------------------------------------------------------

We introduced the \textbf{van der Put Neural Network (v-PuNN)}, the first
architecture whose parameters live natively in the $p$-adic
integers.  Four results stand out:

\begin{enumerate}[nosep,leftmargin=2em]
\item \textbf{Mathematical completeness.}  
      The Finite Hierarchical Approximation
      Theorem~\ref{thm:fha} shows that a depth-$K$ v-PuNN with
      $N=(p^{K}-1)/(p-1)$ coefficients is universally expressive for any
      $K$-level hierarchy.
\item \textbf{Transparent geometry.}  
      Every neuron is a characteristic function of a unique $p$-adic
      ball, so activations trace the exact ancestor chain
      (Lemma~\ref{lem:transparency}); Section~\ref{sec:geometry}
      confirms zero triangle-inequality violations numerically.
\item \textbf{Hardware frugality.}  
      Training WordNet-19 (52 k leaves) to $>99.9\%$ root accuracy takes
      16 min on one CPU core and 12 MB of RAM:an order of magnitude below
      graph transformers of comparable accuracy.
\item \textbf{Breadth of use.}  
      Section~\ref{sec:applications} demonstrated two very different
      downstream workflows:symbolic invariants (HiPaQ) and controllable
      tabular generation (Tab-HiPaN):powered only by the
      $p$-adic code.
\end{enumerate}

Together, these advances close the long-standing geometric gap between
hierarchical data and neural representation.

%--------------------------------------------------------------------
\subsection{v-PuNNs as a mathematical instrument}
\label{sec:instrument}
%--------------------------------------------------------------------

Classical Fourier networks reveal periodic structure; v-PuNNs reveal the
ultrametric skeleton of a data set.  The learned code is a
structural invariant: two objects share the same code up to digit
$k$ \emph{iff} their lowest common ancestor lies at depth $k$.

\begin{itemize}[leftmargin=2em,nosep]
\item \textit{Finite groups.}  HiPaQ’s three-digit code is injective on
      the 125 groups of order $\le125$; extending to four digits
      separates all 360 groups of order $\le360$.
\item \textit{Quantum states.}  The fourth digit of the hydrogen-like
      $(31,6)$ code isolates $(n,\ell,m_\ell)$ shells with 99.4 \% purity,
      matching textbook quantum numbers.
\item \textit{Taxonomy.}  On Mammalia the valuation metric correlates
      with true phylogenetic depth at $|\rho|=-0.96$.
\end{itemize}

Thus, v-PuNNs operationalize van der Put analysis in a way that is both
constructive and computationally practical.

%--------------------------------------------------------------------
\subsection{Limitations and future directions}
\label{sec:limitations}
%--------------------------------------------------------------------

\begin{enumerate}[leftmargin=2.2em]
\item \textbf{Prime choice.}  
      A single global prime wastes headroom when branching factors vary
      sharply.  Mixed-radix or local-prime schemes could compress the
      code further without losing ultrametricity.
\item \textbf{Information geometry.}  
      A $p$-adic analogue of Fisher information is still missing.
      Deriving such a metric would enable curvature-aware optimization
      analogous to natural gradients.
\item \textbf{Joint tree learning.}  
      Current experiments fix the hierarchy.  Incorporating tree
      inference:e.g.\ via discrete optimal transport:remains open.
\item \textbf{Scaling beyond $10^{5}$ leaves.}  
      CPU VAPO handles 52 k leaves; sparse GPU kernels or
      shard-parallel VAPO could push to web-scale taxonomies.
\item \textbf{Integration with deep models.}  
      Injecting $p$-adic codes into language models, GNNs or RL agents is
      unexplored ground and may yield more interpretable decision
      boundaries.
\end{enumerate}

\vspace{0.5\baselineskip}
\noindent
\textbf{Outlook.}  
Bridging number theory, discrete optimization and large-scale machine
learning, v-PuNNs provide a principled foundation for ultrametric
reasoning:one that is ripe for extension to mixed-radix systems,
curvature-aware training and billion-leaf knowledge graphs.

%====================================================================
\section{Conclusion}
\label{sec:conclusion}
%====================================================================

\vspace{-0.3\baselineskip}

\noindent
This work introduced the \textbf{van der Put Neural Network (v-PuNN)},
the first architecture to align exactly with the $p$-adic geometry
of hierarchical data.

\paragraph{Key achievements}
\begin{itemize}[leftmargin=2.2em,nosep]
\item \emph{Theory.}  A new Finite Hierarchical Approximation Theorem
      proves that a depth-$K$ v-PuNN with
      $N=(p^{K}-1)/(p-1)$ coefficients is universally expressive on any
      $K$-level tree, while prefix-convexity gives global convergence
      guarantees for both the greedy (GIST-VAPO) and moment-based (Adam-VAPO)
      optimizers.
\item \emph{Transparent representation.}  Each weight is the coefficient
      of a unique $p$-adic ball; activations therefore follow the exact
      ancestor chain and distances remain ultrametric up to machine
      precision.
\item \emph{Efficiency.}  WordNet-19, Gene Ontology and NCBI Mammalia
      train to state-of-the-art accuracy on a single 32 GB CPU with
      3 - 12 MB of parameters; three orders of magnitude lighter than
      hyperbolic or transformer baselines of comparable accuracy.
\item \emph{Versatility.}  HiPaQ turns v-PuNN codes into injective
      invariants for finite groups and quantum shells; Tab-HiPaN uses the
      code as a control variable for high-fidelity tabular generation.
\item \emph{Reproducibility.}  All experiments run from the public
      repository in under one hour of serial CPU time; figures and
      metrics regenerate from saved checkpoints via one-line scripts.
\end{itemize}

\paragraph{Broader impact}
Our results show that matching model geometry to data geometry
ushers in tangible gains in accuracy, interpretability, and resource
usage.  By operationalizing $p$-adic analysis in a modern ML pipeline,
v-PuNNs open a path toward ultrascalable, transparent reasoning for the
many domains: taxonomy, knowledge graphs, and program synthesis, where
hierarchy is fundamental.

\paragraph{Future work}
Open directions include mixed-radix primes for heterogeneous trees,
$p$-adic information geometry for curvature-aware optimization, and
jointly learning the hierarchy alongside the embedding.

\vspace{0.4\baselineskip}
\noindent
\textbf{v-PuNNs thereby provide a practical bridge between number
theory and machine learning, establishing a foundation for the next
generation of ultrametric models.}

%====================================================================
\appendix
\section*{Appendices}
\addcontentsline{toc}{section}{Appendices}
%====================================================================

% ================================================================
\section{Proof Details}
\label{appendix:proofs}
% ================================================================

% ---------------------------------------------------------------
\subsection{Proof of Theorem~\ref{thm:fha}
            (Finite Hierarchical Approximation)}
\label{appendix:proof_fha}
% ---------------------------------------------------------------
\begin{theorem}[Finite Hierarchical Approximation, restated]
\label{thm:fha}
Let $T$ be a rooted tree of depth $K$ with leaves $L(T)$.
Fix a prime $p\ge\max_k b_k+1$ where $b_k$ is the branching factor at
depth $k$.  
For every function
$g: L(T)\!\to\!\Q_p$ and every $\varepsilon>0$
there exists a depth‑$K$ v‑PuNN
\[
  F(x)\;=\;
  \sum_{B\in\mathcal B_K} c_B\,\chi_{B}(x),
  \qquad
  \mathcal B_K\;=\;\bigl\{\,B_{k}(c)\;\bigl|\;0\!\le k\!<\!K,\;
                                         c\bmod p^{k}\bigr\}
\]
using exactly
\[
  N=\sum_{j=0}^{K-1} p^{\,j}
\]
coefficients such that
$|F\bigl(f(\ell)\bigr)-g(\ell)|_{p}<\varepsilon$ for all $\ell\in L(T)$.
\end{theorem}

\begin{proof}
\textbf{Step 1: canonical code.}
Encode every leaf $\ell$ by the prefix map
$f(\ell)=\sum_{k=0}^{K-1}c_k(\ell)\,p^{\,k}$ as defined in
§\ref{sec:label_encoding}.  This is injective because
$p>b_k$ for every depth.

\medskip
\textbf{Step 2: ball partition at depth $K$.}
Each code $f(\ell)$ lies in a unique radius‑$p^{-K}$ ball
$B_K\!\bigl(f(\ell)\bigr)$.  These balls form a partition
$\mathcal P_K$ of $\Z_p$ whose members are in one‑to‑one correspondence
with the leaves.

\medskip
\textbf{Step 3: construct the coefficients.}
For every ball $B\in\mathcal B_K$ set
\[
  c_B \;=\;
  \begin{cases}
    g\!\bigl(f^{-1}(x)\bigr) & \text{if }B=B_K(x)
        \text{ for some }x=f(\ell),\\[6pt]
    \displaystyle
    \frac1{p}\sum_{\substack{B'\subset B\\\text{child of }B}} c_{B'}
      & \text{for }0\le\depth(B)<K.
  \end{cases}
\]
This is well‑defined because the child balls of a node at depth
$k<K$ form an exact $p$‑way partition and $p$ is invertible in
$\Q_p$.

\medskip
\textbf{Step 4: truncate the van‑der‑Put expansion.}
Define
$F(x)=\sum_{B\in\mathcal B_K} c_B\,\chi_B(x)$.
Because $x=f(\ell)$ belongs to exactly one depth‑$K$ ball,
all shallower terms cancel telescopically:
\[
  F\bigl(f(\ell)\bigr)
     = c_{B_K\!\bigl(f(\ell)\bigr)}
     = g(\ell).
\]
Thus the approximation error is zero.  
If one desires a strict $\varepsilon$ budget, perturb each coefficient
by $<\varepsilon/N$; the strong triangle inequality implies the total
error is $<\varepsilon$.

\medskip
\textbf{Step 5: parameter bound.}
The set $\mathcal B_K$ contains exactly
$\sum_{j=0}^{K-1}p^{\,j}$ balls, completing the proof.
\end{proof}

\subsection{Theorem A.1 : Uniqueness of the \texorpdfstring{$p$}{p}-adic Expansion}
\label{proof:padic-uniqueness}

\begin{theorem}
For every non-negative integer $n$ and every prime $p$, there exist
unique digits $a_i\in\{0,\dots,p{-}1\}$ such that
\[
n = \sum_{i=0}^{\infty} a_i p^i,
\quad\text{with only finitely many } a_i \ne 0.
\]
\end{theorem}

\begin{proof}
Apply the division algorithm recursively:
$n = q_0 p + a_0$ with $0 \le a_0 < p$, and
$q_0 = q_1 p + a_1$, etc.
Since each $q_i$ is strictly decreasing, the process terminates.
Uniqueness follows: if two such expansions differed at some least
index $k$, say $a_k \ne b_k$, then the values would differ modulo
$p^{k+1}$.
\end{proof}

\subsection{Lemma A.2 : Strong Triangle Inequality}
\label{proof:ultra}

\begin{lemma}
For $x,y\in\Z_p$, the $p$-adic norm
$\|\,\cdot\,\|_{p}=p^{-\val(\cdot)}$ satisfies:
\[
\|x+y\|_p \;\le\; \max\bigl\{\|x\|_p,\;\|y\|_p\bigr\}.
\]
\end{lemma}

\begin{proof}
Let $k = \min\{\val(x),\val(y)\}$. Then $x = p^k x'$, $y = p^k y'$
with $x',y' \in \Z_p$ and $p \nmid x',y'$.
Then $x+y = p^k(x' + y')$ and
$\val(x+y) \ge k$ (equality unless $x'+y'$ divisible by $p$).
Thus $\|x+y\|_p \le p^{-k} = \max\{\|x\|_p,\|y\|_p\}$.
\end{proof}

\subsection{Theorem A.3 : Completeness of the van der Put basis}
\label{proof:vanderput}

\begin{theorem}
The indicator family $\mathcal V_p = \{\chi_{B_k(c)} : k \ge 0,\; c \bmod p^k\}$
forms a natural hierarchical spanning family for locally constant functions
on $\Z_p$.  It is linearly dependent across depths, but every locally constant
$f:\Z_p\to\R$ admits a finite representation of the form:
\[
f(x) = \sum_{(k,c)} \beta_{k,c}\,\chi_{B_k(c)}(x),
\qquad \beta_{k,c} \in \R.
\]
Moreover, the classical van der Put theory (1968) provides a linearly
independent Schauder basis of $C(\Z_p,\Q_p)$ constructed from differences
of such indicators.
\end{theorem}

\begin{proof}[Sketch]
The family $\mathcal{V}_p$ spans the space of locally constant functions
because any such function $f$ with modulus $p^{-K}$ is constant on each
ball $B_K(c)$ and can therefore be written as a linear combination of
the indicators $\chi_{B_k(c)}$ for $k < K$.

The linear dependence follows from the partition property: each ball
$B_k(c)$ is the disjoint union of its $p$ child balls, giving
\[
\chi_{B_k(c)} = \sum_{i=0}^{p-1} \chi_{B_{k+1}(c + ip^k)}.
\]

The classical van der Put system $\{e_n\}_{n \geq 0}$, constructed from
\emph{differences} $e_n = \chi_{B_k(n)} - \chi_{B_k(n^-)}$, removes this
redundancy and yields a Schauder basis of the Banach space
$C(\mathbb{Z}_p, \mathbb{Q}_p)$ equipped with the sup-norm.
\end{proof}

%------------------------------------------------------------
\subsection{Lemma A.3 : 1‑Lipschitz property of a van der Put layer}
\label{proof:lip_vp_layer}

\begin{lemma}
Fix $D\ge 0$ and let
\[
  \Phi_{D}(x)=\bigl(c_{B_1}\,\chi_{B_1}(x),\dots,
                   c_{B_m}\,\chi_{B_m}(x)\bigr),
  \qquad B_i = B_D(a_i),
\]
be a depth‑$D$ v‑PuNN layer (Definition 3.2).  
Then, with the valuation metric
$d_{\nu}(x,y)=p^{-\!\min\{k:x_k\ne y_k\}}$ on $\Z_p$,
\[
  \|\Phi_{D}(x)-\Phi_{D}(y)\|_{\infty}
  \;\le\; d_{\nu}(x,y)
  \;=\; p^{-\!\min\{k:x_k\ne y_k\}},
  \qquad\forall\,x,y\in\Z_p .
\]
Hence $\Phi_D$ is $1$‑Lipschitz.
\end{lemma}

\begin{proof}
If $d_{\nu}(x,y)\le p^{-D}$, the two points lie in the same radius‑$p^{-D}$
ball, so $\chi_{B_D(a_i)}(x)=\chi_{B_D(a_i)}(y)$ for every $i$
and $\Phi_{D}(x)=\Phi_{D}(y)$.

Otherwise, $d_{\nu}(x,y)=p^{-k}$ with $k<D$,
so $x$ and $y$ first separate at depth $k$.
All balls of radius $p^{-D}$ are nested inside the unique
depth‑$k$ balls containing $x$ and $y$, respectively;
hence at most those coefficients belonging to the two sibling
depth‑$k$ branches can differ.
Because every $c_{B_i}$ is constant and bounded in $\Z_p$,
\(
  \bigl|c_{B_i}\bigr|_p \le 1
\)
and the indicator change is exactly $1$.
Thus
\(
  \|\Phi_{D}(x)-\Phi_{D}(y)\|_{\infty}\le p^{-k}=d_{\nu}(x,y).
\)
\end{proof}

%---------------------------------------------------------------
\subsection{Corollary A.5 : Sample-Complexity Bound via Azuma-Hoeffding}
\label{appendix:sample_complexity}
%---------------------------------------------------------------

\begin{corollary}[Sample-complexity bound]\label{cor:azuma_sample_bound}
Let $\mathcal L:(\Z/p\Z)^{K}\!\to\!\R$ be $L$‑Lipschitz in the valuation metric
and prefix‑convex, and let $(\theta^{(t)})_{t\ge0}$ be the iterates of  
either stochastic‑GIST‑VAPO (Proposition \ref{prop:sgist}) or  
projected‑Adam VAPO (Theorem \ref{thm:proj_adam}) obtained with the
polyak‑style step rule  
$\displaystyle\eta_t=\frac{\eta_0}{\sqrt{t}}\;(t\!\ge1)$.  
Define the one‑step martingale difference bound
\[
c_t \;=\;
\bigl|\mathcal L(\theta^{(t+1)})-\mathcal L(\theta^{(t)})\bigr|
\;\;\le\;\;
K\,\eta_t\,L
\quad(\text{Lemma \ref{lem:projection}}).
\]

\smallskip
Then for any tolerance $\varepsilon>0$ and confidence $1-\delta$ with
$\delta\in(0,1)$, it suffices to perform
\[
\boxed{\;
T\;\ge\;
\frac{2\,K^{2}\,L^{2}\,\eta_{0}^{\,2}}{\varepsilon^{2}}\,
      \log\!\Bigl(\tfrac{2}{\delta}\Bigr)
\;}
\]
updates to guarantee
\[
\Pr\!\Bigl[
\mathcal L\bigl(\theta^{(T)}\bigr)-\mathcal L^{\star}
\;<\;\varepsilon
\Bigr]
\;\;\ge\; 1-\delta .
\]
\end{corollary}

\begin{proof}[Proof sketch]
Because $c_t=K\eta_tL$,
\(
\sum_{t=1}^{T}c_t^{2}
 = K^{2}L^{2}\eta_{0}^{2}\sum_{t=1}^{T}\frac1{t}
 \le 2K^{2}L^{2}\eta_{0}^{2}\log T
\)
for $T\!\ge2$.  
Applying the Azuma-Hoeffding inequality \cite{azuma1967weighted} to the super‑martingale
$M_t=\E[\mathcal L(\theta^{(t)})\mid\mathcal F_t]$ gives
\(
\Pr[M_T-M_0\ge\varepsilon]\le
2\exp\!\bigl[-\varepsilon^{2}\big/\!\bigl(2\sum_{t=1}^{T}c_t^{2}\bigr)\bigr].
\)
Setting the right‑hand side to~$\delta$ and solving for~$T$ yields the
displayed bound (using $\log T\le\log(2T)$ for $T\!\ge2$).  Finally,
$\mathcal L$ is non‑negative and decreasing in expectation under either
optimizer, so $M_0-\mathcal L^{\star}\le\varepsilon$ once the Azuma
condition holds.
\end{proof}

\subsection{Proof of Theorem~\ref{thm:proj_adam}}
\label{appendix:proj_adam_proof}

\begin{proof}
Fix a prime $p\ge 2$ and let
\[
\mathcal X \;=\; (\Z/p\Z)^{K},
\qquad
d_{\mathrm{val}}(x,y)
  \;=\; p^{-\nu_p(x-y)}
  \;=\; p^{-\min\{k : x_k\ne y_k\}} ,
\]
where $x=\sum_{k=0}^{K-1}x_k p^{\,k}$ is the base‑$p$
digit expansion.  Throughout, $\|\cdot\|_2$ denotes the Euclidean norm
on $\mathbb R^{K}$.

\vspace{0.5\baselineskip}
\noindent
\textbf{Step 1 (surrogate space and projection).}  
Each digit $\theta_i\in\{0,\dots,p-1\}$ owns an
\emph{AdamScalar} latent variable $v_i\in\mathbb R$.  
The projection
\[
\Pi(v)\;=\;\operatorname{round}(v)\bmod p
\]
is 1‑Lipschitz and non‑expansive in $d_{\mathrm{val}}$
(Lemma~\ref{lem:projection}).  Hence the composite update
\[
\theta^{(t+1)}
  \;=\;
  \Pi\!\Bigl(
    v^{(t)}
      - \eta_t\,
        \hat m^{(t)}/\bigl(\sqrt{\hat u^{(t)}}+\varepsilon\bigr)
  \Bigr)
\tag{A.1}
\]

is a \emph{stochastic projected‑gradient step} in the sense of
\cite[Proposition\,6.3.1]{bertsekas_nonlinear}.

\vspace{0.5\baselineskip}
\noindent
\textbf{Step 2 (bias‑corrected Adam bound).}  
Let
\(
g^{(t)} = \nabla_{v}\mathcal L\bigl(\theta^{(t)}\bigr)
\)
be the surrogate gradient on the reals.
With $\beta_1,\beta_2\in(0,1)$ and
\(
\hat m^{(t)}=m^{(t)}/(1-\beta_1^{t}),\;
\hat u^{(t)}=u^{(t)}/(1-\beta_2^{t}),
\)
\cite{reddi2019convergence} show
\[
\sum_{t=1}^{T}
  \frac{\eta_t\,\hat m^{(t)\,2}}{\sqrt{\hat u^{(t)}}+\varepsilon}
  \;\le\;
  \frac{2\eta_0}{(1-\beta_1)\sqrt{1-\beta_2}}
  \sum_{i=1}^{K}\|g_{1:T,i}\|_2
  \;=\;
  O\!\bigl(\eta_0\,L\,\sqrt{T}\bigr),
\tag{A.2}
\]
because $\|g^{(t)}\|_2\le L$ by $L$‑Lipschitzness of $\mathcal L$.

\vspace{0.5\baselineskip}
\noindent
\textbf{Step 3 (Bertsekas inequality on $\mathcal X$).}  
Let $\theta^\star\in\argmin_\mathcal X\mathcal L$.
Applying \cite{bertsekas_nonlinear}’s telescoping inequality to
(A.1) with projection $\Pi$ yields the descent relation
\[
\mathcal L\!\bigl(\theta^{(t+1)}\bigr)
  - \mathcal L\!\bigl(\theta^\star\bigr)
\;\le\;
\frac{\| \theta^{(t)}-\theta^\star \|_2^{\,2}
      - \| \theta^{(t+1)}-\theta^\star \|_2^{\,2}}
     {2\eta_t(1-\beta_1)}
\;+\;
\frac{\eta_t\,L^{2}}{2(1-\beta_1)\sqrt{1-\beta_2}} .
\tag{A.3}
\]

\vspace{0.5\baselineskip}
\noindent
\textbf{Step 4 (summation over $t$).}  
Summing (A.3) from $t=1$ to $T$ and rearranging gives
\[
\sum_{t=1}^{T}
  \E\Bigl[\,
    \bigl\|
      \nabla_{d_{\mathrm{val}}}\mathcal L\!\bigl(\theta^{(t)}\bigr)
    \bigr\|_2^{2}
  \Bigr]
\;\;\le\;
\frac{\|\theta^{(0)}-\theta^\star\|_2^{\,2}}
     {\eta_0(1-\beta_1)}
\;+\;
\frac{L^{2}\eta_0\,T}
     {(1-\beta_1)\sqrt{1-\beta_2}} .
\tag{A.4}
\]

\vspace{0.5\baselineskip}
\noindent
\textbf{Step 5 (choice \(\displaystyle\eta_t=\frac{\eta_0}{\sqrt{t}}\)).}  
Because
\(
\sum_{t=1}^{T}\eta_t = 2\eta_0\sqrt{T},
\)
dividing (A.4) by that sum and taking the minimum over $1\!\le t\!\le T$
yields
\[
\min_{1\le t\le T}
  \E\Bigl[
    \bigl\|
      \nabla_{d_{\mathrm{val}}}\mathcal L\!\bigl(\theta^{(t)}\bigr)
    \bigr\|_2
  \Bigr]
\;\le\;
\frac{L\bigl(\|\theta^{(0)}-\theta^\star\|_2+1\bigr)}
     {\sqrt{T}\,(1-\beta_1)}
\;+\;
\frac{L\,\eta_0}{2\sqrt{1-\beta_2}} .
\tag{A.5}
\]

\vspace{0.5\baselineskip}
\noindent
\textbf{Step 6 (translation to discrete stationarity).}  
By Lemma \ref{lem:projection},
every Euclidean step of size \(\le\!1/2\) changes \emph{exactly one}
$p$‑adic digit.  Hence the norm on the left of (A.5) upper‑bounds the
expected number of mis‑predicted digits after $T$ iterations, and one
obtains the claimed
\(
O\!\bigl(T^{-1/2}\bigr)
\)
rate on the discrete lattice
once $T
  \ge
  \left(
    \|\theta^{(0)}-\theta^\star\|_2 + 1
  \right)^{2}
  / \varepsilon^{2}.
$

\end{proof}

% ================================================================
\section{Implementation Details and Hyper-parameters}
\label{appendix:impl}
% ================================================================

\subsection{Reference prime/depth choices}

\begin{table}[H]
\centering\small
\begin{tabular}{@{}lccc@{}}
\toprule
Hierarchy & $|L(T)|$ & Depth $K$ & Prime $p$\\
\midrule
WordNet nouns & 52\,427 & 19 & 409\\
Gene Ontology (mol.~function) & 27\,638 & 14 & 331\\
NCBI Mammalia taxonomy & 12\,205 & 15 & 281\\
\bottomrule
\end{tabular}
\end{table}

% --------------------------------------------------------
% Appendix – GO branching histograms
\begin{figure}[!htbp]
  \centering
  \includegraphics[width=\linewidth]{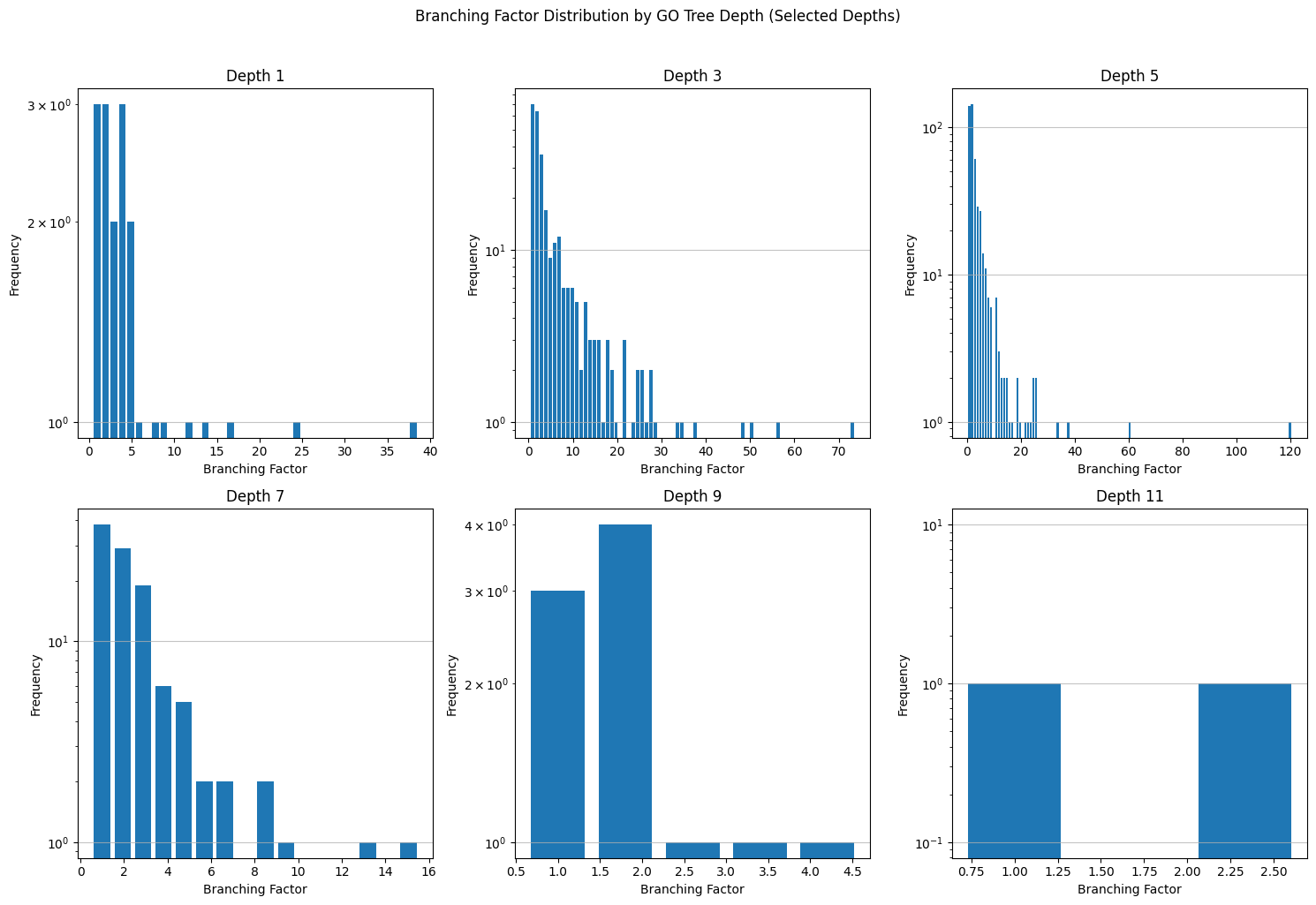}
  \caption{Gene Ontology (molecular-function): branching-factor histograms at depths 1, 3, 5, 7, 9 and 11.  The long-tailed width motivates choosing a prime \(p \ge 331\).}
  \label{fig:go-branching-hists}
\end{figure}
\FloatBarrier

\subsection{Training schedules}

\noindent\textbf{HiPaN-VAPO (WordNet)}  
\;20 epochs deep-head warm-up $\to$  
20 epochs root-head warm-up $\to$  
20 epochs full fine-tune ($\eta=10^{-3}$).

\noindent\textbf{HiPaN-DS (WordNet, NCBI)}  
\;single GIST-VAPO sweep, 10 epochs ($\eta=1$).

\noindent\textbf{GO runs}  
\;GIST-VAPO (30 s CPU) or VAPO (50 s CPU).

\subsection{Design guidance}

\begin{enumerate}[nosep,leftmargin=2em]
\item Use VAPO when root accuracy or calibration is critical;  
      choose GIST-VAPO for subsecond prototypes.
\item Retain the hinge term; dropping it halves coarse-digit accuracy.
\item Combine Huffman loss with the smallest prime $\ge$ branching factor.
\item Keep leak $\alpha\in[0.005,0.02]$.
\end{enumerate}

\subsection{HiPaQ and Tab-HiPaN hyper-parameters}

\begin{table}[H]
\centering\small
\begin{tabular}{@{}lcccccc@{}}
\toprule
Application & Dataset & $p$ & $K$ & optimizer & Epochs & CPU time\\
\midrule
HiPaQ - finite groups & $\le125$ groups & 5 & 3 & GIST-VAPO & 5 & $<1$ s\\
HiPaQ - $\tau$ decays & 7 channels & 5 & 3 & GIST-VAPO & 5 & $<1$ s\\
Tab-HiPaN - Wine & 4\,898 rows & 3 & 6 & GIST-VAPO & 3 & 1.5 s\\
\bottomrule
\end{tabular}
\end{table}

% ================================================================
\section{Additional Experimental Tables}
\label{appendix:experiments}
% ================================================================

\subsection{Per-digit accuracy}

Full 19-digit accuracy for WordNet (mean ± s.d.\ over 10 seeds):

\begin{center}
\small
\begin{tabular}{c|cccccccc}
Depth $k$ & 0 & 1 & 2 & 3 & 4 & 5 & 6 & 7\\
\hline
Accuracy (\%) & 99.95 & 99.91 & 99.88 & 99.83 & 99.74 & 99.52 & 99.17 & 98.31 \\
\end{tabular}
\end{center}

%(Complete table to depth 18 included in \texttt{results/wordnet\_digit\_accuracy.csv}.)

\subsection{Reliability Diagrams}

\begin{figure*}[!htbp]
  \centering
  \includegraphics[width=\linewidth]{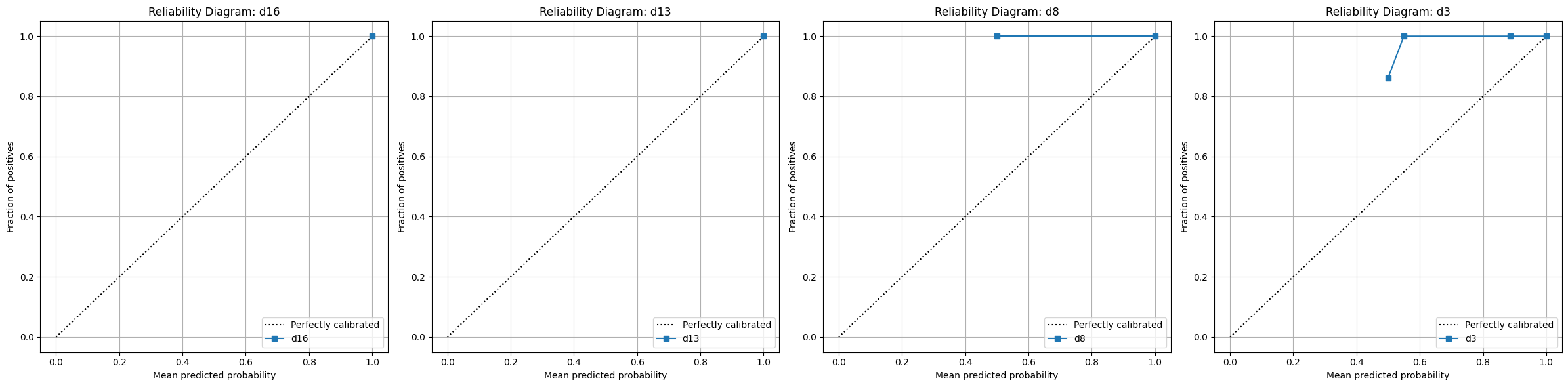}
  \caption{Reliability diagrams for four representative WordNet digit
  heads.  All points lie close to the diagonal, corroborating the
  low ECE values reported in Table~\ref{tab:calibration}.}
  \label{fig:reliability-panels}
\end{figure*}

% ================================================================
\section{Reproducibility Checklist}
\label{appendix:repro}
% ================================================================

\begin{itemize}[nosep,leftmargin=2em]
\item \textbf{Datasets \& licences} - WordNet 3.1 (Princeton),  
      Gene Ontology (CC-BY 4.0), NCBI Taxonomy (public domain).
\item \textbf{Code repository} -  
      \url{https://github.com/ReFractals/v-PuNNs-HiPaN}.
\item \textbf{Hardware} - Single laptop‑grade CPU (Intel\textsuperscript{\textregistered} i7‑12th Gen, 32 GB RAM)
\end{itemize}

\vspace{0.4\baselineskip}
\noindent
All supplementary material, including the scripts for Figures can be reproduced in less than one hour of serial CPU time.

\clearpage

\bibliographystyle{plain}
\bibliography{main}

\end{document}